\documentclass[final,onefignum,onetabnum]{siamonline220329}

\usepackage{graphicx}
\usepackage{epstopdf}
\usepackage[utf8]{inputenc}
\usepackage[english]{babel}
\usepackage{amsmath}
\usepackage{geometry}
\usepackage{wrapfig}
\usepackage{amssymb}
\usepackage{aliascnt}
\usepackage{breakcites}
\usepackage{caption}
\usepackage{subcaption}
\usepackage{tikz}
\usepackage{mathtools}
\usepackage{lmodern}
\usepackage{cleveref}
\usepackage{algorithmic}
\usepackage{appendix} 
\usepackage{multicol}
\usepackage[export]{adjustbox}
\usepackage{enumerate}
\usepackage{proof-at-the-end}
\usetikzlibrary{arrows}
\usetikzlibrary{decorations.markings}
\graphicspath{ {../../codes/data/results/} }
\usepackage{wrapfig}
\usepackage{amsfonts}
\usepackage{algorithmic}
\ifpdf
  \DeclareGraphicsExtensions{.eps,.pdf,.png,.jpg}
\else
  \DeclareGraphicsExtensions{.eps}
\fi

\usepackage{enumitem}
\setlist[enumerate]{leftmargin=.5in}
\setlist[itemize]{leftmargin=.5in}

\newcommand{\R}{\mathbb{R}}
\newcommand{\N}{\mathbb{N}}
\newcommand{\Z}{\mathbb{Z}}

\newcommand{\Mc}{\mathcal{M}}
\newcommand{\Wc}{\mathcal{W}}

\newcommand{\Ic}{\mathcal{I}}
\newcommand{\Lc}{\mathcal{L}}
\newcommand{\Ec}{\mathcal{E}}
\newcommand{\Uf}{\mathfrak{U}}

\newcommand{\Df}{\mathfrak{D}}
\newcommand{\NN}{\mathcal{NN}}
\newcommand{\Rf}{\mathfrak{R}}
\newcommand{\D}{\mathrm{D}}
\newcommand{\wrt}{\:\mathrm{d}}
\newcommand{\w}{\mathfrak{w}}

\DeclareMathOperator*{\supp}{supp}

\DeclareMathOperator*{\TV}{TV}
\DeclareMathOperator*{\BV}{BV}

\DeclareMathOperator*{\PC}{PC}
\DeclareMathOperator*{\BN}{BN}

\usepackage{amsopn}

\newsiamremark{remark}{Remark}
\newsiamremark{hypothesis}{Hypothesis}
\newsiamremark{assumption}{Assumption}
\newsiamremark{ex}{Example}
\crefname{hypothesis}{Hypothesis}{Hypotheses}
\newsiamthm{claim}{Claim}
\newsiamthm{defprop}{Definition and Proposition}

\headers{On the Regularity of Images Generated by CNNs}{A. Habring and M. Holler}

\title{A Note on the Regularity of Images Generated by Convolutional Neural Networks}

\author{Andreas Habring\thanks{Department of Mathematics and Scientific Computing, University of Graz, Austria
  (\email{andreas.habring@uni-graz.at}).}
\and Martin Holler\thanks{Department of Mathematics and Scientific Computing, University of Graz, Austria
  (\email{martin.holler@uni-graz.at}, \url{https://imsc.uni-graz.at/hollerm/}).}}

\usepackage{amsopn}

\definecolor{myblue}{RGB}{30, 76, 132}

\begin{document}

\maketitle

\begin{abstract}
The regularity of images generated by convolutional neural networks, such as the U-net, 
generative networks, or the deep image prior, is analyzed.
In a resolution-independent, infinite dimensional setting, it is shown that such images, represented as functions, are always continuous and, in some circumstances, even continuously differentiable, contradicting the widely accepted modeling of sharp edges in images via jump discontinuities. While such statements require an infinite
dimensional setting, the connection to (discretized) neural networks used in practice is made by considering the limit
as the resolution approaches infinity. 
As practical consequence, the results of this paper in particular provide analytical evidence that basic $L^2$ regularization of network weights might lead to over-smoothed outputs.
\end{abstract}



\section{Introduction}
This paper is concerned with convolutional neural network (CNN) architectures frequently used in imaging and inverse problems, where images are obtained as the output of the CNN. %
We distinguish the following two types of models:
i) End-to-end imaging methods \cite{ronneberger2015u,jain2008natural,jin2017deep,isola2017image,mao2016image,mohan2019robust,ledig2017photo}, where CNNs represent an image transformation, such as a mapping from a noisy image to a clean one, or from an image to its segmentation map. ii) Generative approaches, such as generative adversarial networks (GANs) \cite{goodfellow2014generative}, which synthesize images from hidden latent variables. Besides the learning of image distributions and subsequent sampling \cite{bora2017compressed,skandarani2021gans}, the latter are also frequently used as priors in inverse imaging problems, see, e.g., \cite{deep_image_prior,obmann2019sparse,habring2022generative} or \cite{asim2020invertible,heckel2019denoising,heckel2019deep,hand2018phase,bora2017compressed} for works on \emph{non-convolutional} generators. For a review on CNNs used for the solution of inverse problems see \cite{arridge2019solving}.

Our focus lies in analyzing the regularity of images generated by such CNNs. We show that, when $L^2$ regularization of the convolution kernels is used in training \cite{genzel2022near,jiang2020two,dong2017automatic}, these images are at least continuous and, in some circumstances, even continuously differentiable. This is an undesirable property, since sharp edges represented by jump discontinuities are a key feature of natural images \cite{rudin1987images,chambolle2010introduction}. While such smoothing properties of CNNs are well-known in practice, our work provides the theoretical foundation for such empirical observations and consequentially suggests remedies: While for reasons of well-posedness, regularization is indispensable, our results suggest to refrain from using basic $L^2$ regularization of the network weights or to suitably adapt the network architecture when image data is the network output.

There are several works discussing regularity of CNNs as mappings between finite dimensional spaces \cite{mallat2016understanding,rahaman2019spectral} and on their noise removal performance \cite{heckel2019denoising}. There is, however, hardly any literature on CNNs as mappings between function spaces \cite{habring2022generative,obmann2020deep} and we are not aware of works considering infinite resolution limits of existing network architectures. 

\paragraph{Outline of the paper}
Since statements about regularity of images are only meaningful in an infinite dimensional function-space setting, in \Cref{sec:continuous} we introduce a class of \emph{function space} CNNs for which the mentioned regularity results are provided. In \Cref{sec:discrete} we formally define the practically used discrete CNNs and then show that these converge to the function space CNNs as the resolution approaches infinity, establishing the connection of our results to applications. At last, providing also experimental results, in \Cref{sec:numerical_results} we relate our findings to the deep image prior \cite{deep_image_prior} and to end-to-end methods.

\section{Convolutional Neural Networks in Function Space}\label{sec:continuous}
In the following $\Omega, \Sigma \subset \mathbb{R}^2$ are always \emph{rectangular domains}, i.e, non-empty sets of the form $(a,b) \times (c,d)$ with $a,b,c,d \in \R$, with $0\in\Sigma$. We denote $\Omega_\Sigma \coloneqq \Omega-\Sigma = \left\{x- y\; \middle| \; x\in\Omega, \; y\in\Sigma \right\}$ and the Hölder conjugate exponent of $p \in [1,\infty]$ as $p' \in [1,\infty]$. For $\theta\subset \mathbb{R}^d$ and a function $f:\theta \rightarrow \R$ we denote its zero extension as $\tilde{f}$. We denote the indicator function on a set $M$ as $\Ic_M$, that is, $\Ic_M(x)=0$ if $x\in M$ and $\infty$ otherwise. The space of continuous and compactly supported functions is defined as  $C_c(\theta)\coloneqq \{f\in C(\theta)\;|\; \supp(f)\subset \theta\}$ where $\supp(f) = \overline{\{x\in \theta\;|\;f(x)\neq 0\}}$. Further, we define $C_0(\theta)$ as the closure of $C_c(\theta)$ with respect to the uniform norm, the space $(\mathcal{M}(\theta),\|\; .\; \|_\mathcal{M})$ as the dual of $C_0(\theta)$, and the space $(\mathcal{M}(\overline{\theta}),\|\; .\; \|_\mathcal{M})$ as the dual of $C(\overline{\theta})$. Note that $\Mc(\theta)$ (resp. $\Mc(\overline{\theta})$) coincides with the space of finite Radon measures and $\|\; .\; \|_\mathcal{M}$ with the total variation norm, whenever $\theta$ (resp. $\overline{\theta}$) is a locally compact, separable metric space \cite[Theorem 1.54, Remark 1.57]{bvfunctions}. We denote the space of functions of bounded variation on $\theta$ as $\BV(\theta)\coloneqq\{u\in L^1(\theta)\; | \; \D u \in \Mc(\theta)\}$, where $\D u$ denotes the distributional derivative. For $u\in \BV(\theta)$ we define the total variation $\TV(u)\coloneqq |\D u|(\theta)$; for details see \cite{bvfunctions}. We use the usual notation $W^{k,p}(\theta)$ for the Sobolev space of order k with exponent p and, in the case $p=2$, we denote $H^k(\theta)\coloneqq W^{k,p}(\theta)$.

The goal of this section is to provide an infinite dimensional model for CNNs which corresponds to the limit of discrete CNNs for infinite resolution. In order to do so, we have to introduce analogous versions of the typical building blocks of discrete CNNs in the infinite dimensional setting. Precisely, we will define the notions of convolution and reflection padding for functions and, building on these preliminaries, rigorously define function space CNNs in \Cref{defin:fscnn}. Afterwards we will prove our main regularity result in \Cref{thm:regularity}.

\begin{defprop}{(Convolution)}\label{lem_conv_measures}\
Let $p\in (1,\infty]$, $w\in L^{p}(\Sigma)$ and $u\in L^1(\Omega_\Sigma)$. We define the (valid) convolution $u*w\in L^p(\Omega)$ as
\begin{equation}\label{eq_lem_conv_meas}
\begin{gathered}
u*w(x) = \int_\Sigma w(y)u(x-y) \wrt y.
\end{gathered}
\end{equation}
It holds that $\|u*w\|_p \leq \|u\|_1\|w\|_{p}$ and, if $u\in L^{p'}(\Omega_\Sigma)$, then $u*w\in C(\overline{\Omega})$.

\begin{proof}
The norm estimate follows from Young's inequality and the proof for $u*w\in C(\overline{\Omega})$ can be found in \cite[Theorem 3.14]{bredies2018mathematical}.
\end{proof}
\end{defprop}
\begin{remark}\label{rmk:conv_measure}
Via duality, one can also define the convolution for $u\in \Mc(\Omega_\Sigma)$ or $u\in \Mc(\overline{\Omega_\Sigma})$ \cite[Section 3.1]{Chambolle2020} such that $u*w\in L^p(\Omega)$ and $\|u*w\|_p \leq \|u\|_\Mc\|w\|_{p}$.
\end{remark}
In our function space model for CNNs we combine the convolution with reflection padding, which is also well defined for $L^p$ functions without the need of trace operators and preserves regularity. 
To this aim, consider first the domain $\theta=(a,b) \times (c,d)$. For a function $u:\overline{\theta}\rightarrow \R$, we define the reflection $\overline{u}$ on $[a,b]\times [c-(d-c),d]$ via $\overline{u}(x_1,x_2) = u(x_1,|x_2-c|+c)$. For a measure $\mu\in\Mc(\overline{\theta})$ we define the reflection $\overline{\mu}\in\Mc([a,b]\times [c-(d-c),d])$ via $\overline{\mu}(A) = \mu(A\cap \overline{\theta})+\mu(\{(x_1,(c-x_2)+c)\;|\; (x_1,x_2)\in A\cap [a,b]\times [c-(d-c),c)\})$. Repeating this reflection technique along different axes allows to extend $\overline{u}$ (resp. $\overline{\mu}$) to arbitrary large domains in $\R^2$. 
\begin{definition}{(Extensions by reflection)}
For $\Omega \supset \theta $, we define the extensions by reflection of $u:\overline{\theta} \rightarrow \R$ and $\mu \in \Mc(\overline{\theta})$ to $\Omega$ via the linear operator $\Rf_{\theta,\Omega}$ as
\begin{equation}
\Rf_{\theta,\Omega} u \coloneqq \overline{u}|_\Omega,\quad \Rf_{\theta,\Omega} \mu \coloneqq \overline{\mu}|_\Omega.
\end{equation}
\end{definition}
It's easy to see that if $u\in C(\overline{\theta})$ or $u\in L^p(\theta)$ then also $\Rf_{\theta,\Omega} u\in C(\overline{\Omega})$ or $\Rf_{\theta,\Omega} u\in L^p(\Omega)$, respectively. Further the reflection preserves the regularities $W^{1,p}$ \cite[Lemma 9.2]{brezis2011functional} and $\BV$ \cite[Proposition 3.21]{bvfunctions}.
\begin{remark}
Constant padding would also be feasible, but leads to a loss of regularity: While it preserves $\BV$ and $L^p$ regularity, it prevents continuity or $W^{1,p}$ regularity in general. In particular, this would preserve the result of case 1 of \Cref{thm:regularity} below, but weaken the result of case 2 of \Cref{thm:regularity} below to an output in $W^{1,p}(\overline{\Omega}) \cap C(\overline{\Omega})$.
\end{remark}

We can now define convolutional neural networks as mappings between function spaces.
\begin{definition}{(Function space CNN)}\label{defin:fscnn}
Let $L\in \N$ be a given number of layers, $\sigma: \R\rightarrow \R$ an activation function and $p\in (1,\infty]$.
Given further rectangular domains $\Omega_L=\Omega$ and $\Omega_l$, $\Sigma_l$,  $l=0,\ldots,L-1$ denote $\Omega_{l-1}' = \Omega_l-\Sigma_{l-1}$. Let the input $\mu\in \Mc(\overline{\Omega_0})$, the convolution kernels $w = (w^{l})_l\in L^{p}(\Sigma_0)\times\dots \times L^{p}(\Sigma_{L-1})\eqqcolon \Wc^{p} (\Sigma)$, and the biases $b=(b^l)_l\in\R^L$. We define the function space convolutional neural network \eqref{eq:continuous_neural_network} as $\NN: \Mc(\overline{\Omega_0})\times \Wc^p (\Sigma)\times \R^L\rightarrow L^p(\Omega)$
\begin{equation}\label{eq:continuous_neural_network}
\tag{fsCNN}
\begin{aligned}
\NN(\mu,w,b) &= (\Rf_{\Omega_{L-1},\Omega_{L-1}'}{\mu^{L-1}}) * w^{L-1}+b^{L-1} \quad\text{where}\\
\mu^0&=\mu\\
\mu^{l+1} &= \sigma((\Rf_{\Omega_l,\Omega_l'}{\mu^l}) * w^{l}+b^l),\; l=0,\dots , L-2
\end{aligned}
\end{equation}
where the activation function $\sigma$ is applied point-wise, that is, for a function $u$, $\sigma(u)(x)\coloneqq \sigma(u(x))$.
\end{definition}
\begin{remark}\
\begin{itemize}
\item For the sake of brevity, we use the following simplifications. We only use single-channel convolutions (one convolution filter per layer), the same activation function in all layers and no skip connections in the model. The first two assumptions can be generalized without affecting the analysis, while in case of skip connections with appropriate cropping, our theoretical results remain true if there is no connection from input to output of the network with less than two convolutional layers, which is usually the case in practice. Batch normalization, for which the necessary results can be found in \Cref{sec:appendix_miscellaneous}, \Cref{defin:batch_norm} and \Cref{lem:batch_norm}, can be included as well.

\item As shown in \Cref{sec:discrete} below, down- and upsampling layers, which are commonly used in practice, vanish for infinitely fine discretization and \eqref{eq:continuous_neural_network} is, indeed, a proper function space model for discrete CNNs.
\end{itemize}
\end{remark}

We refer the reader to \Cref{sec:appendix_regularity} for several smaller results concerning regularity in particular of the convolution and the application of the activation function. While versions of these results are rather standard, they play a crucial role in the proof of the following main regularity result \Cref{thm:regularity}.

\begin{theorem}(Regularity)\label{thm:regularity}
Given $\NN$ as in \eqref{eq:continuous_neural_network} with $\sigma$ Lipschitz continuous, $p=2$, and at least two convolutional layers. Then $\NN(\mu,w,b)\in C(\overline{\Omega})$ and, if additionally $\mu\in \BV(\Omega_0)$, we obtain $\NN(\mu,w,b)\in C^1(\overline{\Omega})$.
\begin{proof}
\textbf{Case 1:} If $\mu\in \Mc(\overline{\Omega_0})$ also $\Rf_{\Omega_0,\Omega_0'}{\mu}\in \Mc(\overline{\Omega_0'})$. Using Lipschitz continuity of $\sigma$, boundedness of all domains, and \cref{lem_conv_measures}, we find $\mu^1\in L^2(\Omega_1)$ and $\mu^2\in C(\overline{\Omega_2})$.

\textbf{Case 2:} If $\mu\in \BV(\Omega_0)$ also $\Rf_{\Omega_0,\Omega_0'}{\mu}\in \BV(\Omega_0')$. Application of \cref{lem:derivative_convolution} and \cref{lem:activation_regularity} yields $\mu^1 \in H^1(\Omega_1)$. Then, by \cref{lem_conv_measures}, the pre-activation $\hat{\mu}^2\coloneqq \Rf_{\Omega_1,\Omega_1'}{\mu^1} * w^{1}+b^1$ and its derivative $\D\hat{\mu}^2= (\D\Rf_{\Omega_1,\Omega_1'}{\mu^1}) * w^{1}$ are in  $C(\overline{\Omega_2})$. Thus, by \Cref{lem:continuous_weak_derivative}, $\hat{\mu}^2 \in C^1(\overline{\Omega_2})$.

In both cases, further layers of the network maintain the regularity, yielding the desired result.
\end{proof}
\end{theorem}

\begin{remark}(Input regularity)\
\begin{itemize}
\item In end-to-end applications the CNN is typically fed the given corrupted data. For instance in the case of Gaussian denoising the input is the corrupted image, deteriorated by additive Gaussian noise which is usually assumed to be contained in $L^2(\Omega_0)$. This is also reflected in the wide spread application of the 2-norm -- or mean squared error (MSE) loss -- for the data fidelity in denoising tasks \cite{zhang2018ffdnet}. Our general model is even less restrictive, allowing $\mu\in\Mc(\overline{\Omega_0})$ to be a Radon measure as in \cite{habring2022generative}. This way, the input of the network can even be, e.g., a Dirac delta \cite[Sections 2,3]{Chambolle2020}. \Cref{thm:regularity} guarantees a continuous output in this case. Let us briefly elaborate on the advantages of allowing Radon measure valued inputs. In many applications -- especially (convolutional) sparse coding based methods \cite{Chambolle2020,zhang2017convolutional} -- the sparsity inducing 1-norm penalty is applied to the input in the discrete setting. Using the space $L^1$ as an infinite dimensional model will, however, often lead to a minimization problem which cannot be proven to admit a solution due to the analytically disadvantageous nature of $L^1$. The preferred infinite dimensional model for the discrete 1-norm therefore is the total variation norm and the corresponding function space the space of Radon measures which, contrary to $L^1$, is the dual of a separable Banach space and thus allows for the application of the theorem of Banach-Alaoglu. For an example of such a model with proven existence of minimizers we refer the reader to \cite{habring2022generative}.
\item On the other hand, in some cases, such as end-to-end image segmentation, the input of the CNN is a clean image. While there is some contrary evidence \cite{gousseau2001natural}, the space of functions of bounded variation $\BV(\Omega_0)$ as a model for natural images has been the starting point of many successful imaging methods \cite{chan1998total,rudin1987images,rudin1992nonlinear,chambolle2010introduction}, in particular any method utilizing total variation regularization, as $\BV$ is the natural domain of definition of $\TV$ \cite{bvfunctions}. Therefore in case 2 of \Cref{thm:regularity} we consider $\mu\in \BV(\Omega_0)$ showing that in this setting the output can even be expected to be of $C^1$ regularity.
\end{itemize}
\end{remark}

\begin{remark}(Higher or lower output regularity)\
\begin{itemize}
\item Without further assumptions one cannot expect to obtain higher regularity, that is, there is no $s>0$ such that the convolution of any function in $H^k$ with any function in $L^2$ is in $H^{k+s}$: A function $u\in  L^2((-L,L)^2)$ is in $H^k((-L,L)^2)$ if and only if $\sum_{n\in\Z^2} (1+|n|^2)^k |\hat{u}_n|^2<\infty$, where $(\hat{u}_n)_n$ are the Fourier coefficients of $u$. Let $s>0$ be arbitrary and consider $\hat{u}_{n_m}=\frac{1+|m|^2}{(1+|n_m|^2)^{\frac{k+s}{2}}}$ and $\hat{u}_l=0$ for $l\neq n_m$ where $n_m=(n_1(m_1),n_2(m_2))$ is a subsequence increasing fast enough such that $u\in H^k((-L,L)^2)$. Let $w \in L^2((-L,L)^2)$, $\hat{w}_{n_m}=1/(1+|m|^2)$ and $\hat{w}_l=0$ for $l\neq n_m$. For the convolution with periodic boundary conditions we obtain $(\widehat{u*w})_{n_m}=4L^2 \hat{u}_{n_m}\hat{w}_{n_m}=4L^2/(1+|n_m|^2)^{\frac{k+s}{2}}$ and therefore $u*w\notin H^{k+s}((-L,L)^2)$.
\item Naturally, \Cref{thm:regularity} indicates strategies to decrease the regularity of the output. The crucial gain in regularity results from the last case in \Cref{lem_conv_measures} which states that the convolution $u*w$ of two functions $u\in L^{p'}$, $w\in L^p$ from dual $L^p$ spaces is continuous. If, however, $\frac{1}{p}+\frac{1}{q}>1$ continuity of the convolution cannot be expected in general. Thus, if we choose $p\in (1,2)$ in \Cref{defin:fscnn}, the output of \eqref{eq:continuous_neural_network} will frequently not be continuous. For $p>2$ the results remain true due to the natural embedding $L^{p_1}\subset L^{p_2}$ for $p_1>p_2$ on bounded domains. Choosing $p=1$ will lead to several complications down the road. Most prominently, issues will arise when investigating training of the CNN as $L^1$ is not reflexive, which is a necessary ingredient for the existence of a solution to the training problem, cf. \Cref{lem:gamma_convergence}.
\item Reducing regularity even further, one could also model the weights as Radon measures, for which the convolution can be generalized via the identification of the space of Radon measures with the dual of $C_0$. The convolution of two Radon measures then turns out to be a radon measure again. This route, however, leads to complications regarding the application of an activation function $\sigma$, as it is unclear, how to generalize the point-wise application of $\sigma$ to measures instead of functions.
\end{itemize}
\end{remark}

\section{Application to Neural Networks Operating on Pixel Grids}\label{sec:discrete}

In this section we bridge the gap between the above introduced function space CNNs and the discrete CNNs used in practice. To do this we first show that discrete CNNs can be regarded as finite dimensional approximations of function space CNNs of a specific structure. Afterwards we consider the limit of these approximations for infinite resolution.

\subsection{Discrete Convolutional Neural Networks}
We discretize $\R^2$ with a grid of resolution $h>0$. For matrices $U \in \R^{N \times M}$, $W \in \R^{n \times m}$, we denote the \emph{valid} convolution as $U*W \in \R^{(N - n + 1  ) \times ( M - m + 1 ) }$, $(U*W)_{i,j}=\sum_{k,l}h^2U_{i+n-k,j+m-l}W_{k,l}$. The extension by reflection in the discrete case is defined analogously to the continuous case (see \cite[Section 3.3.3]{bredies2018mathematical}) and denoted by an operator $\Rf$. We define a downsampling operator by a factor $s\in\N$ with an averaging function $f:\R^{s\times s}\rightarrow \R$ as $\Df_s: \R^{sN\times sM} \rightarrow \R^{N\times M}$
\begin{equation}
\begin{aligned}
(\Df_s(U))_{i,j} &= f( (U_{si+k,sj+l})_{k,l=0}^{s-1} ).
\end{aligned}
\end{equation}
Examples for different choices of $f$ are max pooling or average pooling. Upsampling by a factor $s\in\N$ with an interpolation function $g:\R^{2\times 2}\times \rightarrow \R^{s\times s}$ is defined as $\Uf_s: \R^{N\times M} \rightarrow \R^{Ns\times Ms}$
\begin{equation}
\begin{aligned}
((\Uf_s(U))_{si+k,sj+l})_{k,l=0}^{s-1} &= g( U_{i,j},U_{i+1,j},U_{i,j+1},U_{i+1,j+1}),
\end{aligned}
\end{equation}
where we use constant extension of $U$ at the boundary. Examples for different choices of $g$ are constant upsampling or bilinear upsampling.

\begin{definition}{(Discrete CNN)}
Let $L  \in \N$ be the number of layers and $ N_l,M_l$ and $n_l,m_l \in \N$ for $l=0,\ldots,L-1$ the size of the activations and convolution kernels, respectively, at the $l$-th layer. To this aim let $(N_{l-1}',M_{l-1}') = (N_l + n_{l-1} - 1, M_l+m_{l-1}-1)$ and $\Rf_l:\R^{N_l\times M_l}\rightarrow\R^{N_l'\times M_l'}$ be the operator performing extension by reflection to the wanted size. We define the discrete CNN $\NN_d: \R^{N_0\times M_0}\times \prod_l \R^{n_l\times m_l} \times \R^L \rightarrow \R^{N_L\times M_L}$,
\begin{equation}\label{eq:discrete_neural_network}
\tag{dCNN}
\begin{aligned}
\NN_d(U,W,b) &= \Df_{s_{L-1,D}}((\Rf_{L-1}{\Uf_{s_{L-1,U}}U^{L-1}}) * W^{L-1}+b^{L-1})\\
U^{1} &= \Df_{s_{0,D}}\sigma((\Rf_0{U}) * W^0+b^0)\\
U^{l+1} &= \Df_{s_{l,D}}\sigma((\Rf_l{\Uf_{s_{l,U}}U^l}) * W^l+b^l),\; l=1,\ldots,L-2.\\
\end{aligned}
\end{equation}
\end{definition}
The down- or upsampling in each layer can be omitted by setting $s_{l,D}$ or $s_{l,U}$ to one, respectively. The frequently used strided convolutions can be realized in this model as a regular convolution with subsequent downsampling.

\subsection{A Finite Dimensional Approximation} Inspired by a finite element type approach, we now derive a finite dimensional approximation of \eqref{eq:continuous_neural_network} which is equivalent to \eqref{eq:discrete_neural_network}. We denote the grid points $x^h_{i,j} = (ih,jh)$ for $(i,j)\in\mathbb{Z}^2$ and the squares $Q^h_{i,j} = x^h_{i,j} + [-h/2, h/2)^2$. Further we define $\PC^h(\R^2)$ as the space of functions $\R^2\rightarrow \R$ that are constant on $Q^h_{i,j}$ for all $i,j$ and, for an open set $\theta\subset\R^2$, $\PC^h(\theta) = \{u|_\theta\; | \; u\in \PC^h(\R^2)\}$.
We assume that the height and width of all occurring domains $\theta$ are integer multiples of the resolution $h$. For functions $u$ that are defined pointwise, we introduce the linear projection $P^hu$ onto $\PC^h(\theta)$ as%
\begin{equation}\label{eq:projection_pc}
\begin{aligned}
P^hu = \sum_{x^h_{i,j}\in\theta} u(x^h_{i,j}) \; \chi_{Q^h_{i,j}}.
\end{aligned}
\end{equation}
Note that $\sup_x |P^hu(x)|\leq \sup_x |u(x)|$ and that, if $u$ is uniformly continuous, $P^h u\rightarrow u$ uniformly as $h\rightarrow 0$. As a consequence $\PC^h(\theta)$ is dense in $L^p(\theta)$ for all $1\leq p<\infty$ if $\theta$ is bounded and we use these spaces as approximation spaces for $L^p(\theta)$. Although, for bounded $\theta$, $\PC^h(\theta)$ is finite dimensional, it is still a space of continuously defined functions and therefore amenable to the \emph{continuous} convolution as defined in \Cref{lem_conv_measures}. We define a downsampling operator $\Df_s^h: \PC^h(\theta) \rightarrow \PC^{sh}(\theta)$
\begin{equation}
\begin{aligned}
u^h=\sum_{k,l} U^h_{k,l} \;\chi_{Q^h_{k,l}} &\mapsto \sum_{k,l} \Df_s(U^h)_{k,l} \; (\sum_{i,j=0}^{s-1}\chi_{Q^h_{sk+i,sl+j}}).
\end{aligned}
\end{equation}
and an upsampling operator as $\Uf_s^{sh}: \PC^{sh}(\theta) \rightarrow \PC^{h}(\theta)\\$
\begin{equation}
\begin{aligned}
u^h = \sum_{k,l} U^h_{k,l} \;(\sum_{i,j=0}^{s-1}\chi_{Q^h_{sk+i,sl+j}})&\mapsto \sum_{k,l} \sum_{i,j=0}^{s-1}\Uf_s(U^h)_{sk+i,sl+j} \chi_{Q^h_{sk+i,sl+j}}.
\end{aligned}
\end{equation}
\begin{definition}{(Approximated CNN)}
With $\Wc_h(\Sigma)\coloneqq \PC^{h_0}(\Sigma_0)\times\dots \times \PC^{h_{L-1}}(\Sigma_{L-1})$ where $h_l$ is the resolution at layer $l$ we define a finite dimensional approximation with resolution $h$ of \eqref{eq:continuous_neural_network}, $\NN_h: \PC^{h_0}(\Omega_0)\times \Wc_h(\Sigma)\times\R^L\rightarrow L^2(\Omega)$ as
\begin{equation}\label{eq:approx_neural_network}
\tag{aCNN}
\begin{aligned}
\NN_h(u^h,w^h,b^h) &= \Df^{h_{L-1}}_{s_{{L-1},D}}P^{h_{L-1}}((\Rf_{\Omega_{L-1},\Omega_{L-1}'}\Uf^{s_{{L-1},U}h_{L-1}}_{s_{{L-1},U}}u^{h,{L-1}}) * w^{h,{L-1}}+b^{h,{L-1}})\\
u^{h,1} &= \Df^{h_0}_{s_{0,D}}P^{h_0}\sigma((\Rf_{\Omega_0,\Omega_0'}u^h) * w^{h,0}+b^{h,0})\\
u^{h,l+1} &= \Df^{h_l}_{s_{l,D}}P^{h_l}\sigma((\Rf_{\Omega_l,\Omega_l'}\Uf^{s_{l,U}h_l}_{s_{l,U}}u^{h,l}) * w^{h,l}+b^{h,l}),\; l=1,\ldots,L-2.\\
\end{aligned}
\end{equation}
\end{definition}
Note that the resolutions on different layers can be computed via $h_l = h_0\Pi_{i=1}^l\frac{s_{i-1,D}}{s_{i,U}}$ for $l\geq 1$, which follows from the definitions of the up- and downsampling operators. To relate \eqref{eq:approx_neural_network} and \eqref{eq:discrete_neural_network} let us first consider the convolution. With $u^h = \sum_{k,l} U^h_{k,l} \;\chi_{Q^h_{k,l}}$ and $w^h = \sum_{k,l} W^h_{k,l}\; \chi_{Q^h_{k,l}}$ the continuous convolution $u^h*w^h$ evaluated at the grid point $x_{n,m}^h$ reads as
\begin{equation}
\begin{aligned}
u^h*w^h(x^h_{n,m}) &= \sum_{k,l,i,j} U^h_{k,l}W^h_{i,j} \chi_{Q^h_{k,l}}*\chi_{Q^h_{i,j}}\\
&= \sum_{k,l,i,j} U^h_{k,l}W^h_{i,j} \underbrace{\int \chi_{Q^h_{k,l}}(x^h_{n,m}-y)*\chi_{Q^h_{i,j}}(y)\; \wrt y}_{= \begin{cases}
h^2\quad \text{ if }x^h_{n,m}-x^h_{i,j} = x^h_{k,l}\\
0\quad \text{ otherwise}
\end{cases} }
= h^2 \sum_{i,j} U^h_{n-i,m-j}W^h_{i,j}.
\end{aligned}
\end{equation}
Hence, $u^h*w^h(x^h_{n,m}) = (U^h*W^h)_{n,m}$, where the right-hand side denotes the discrete convolution. Since $u^h*w^h$ is also defined between the sampling points $x^h_{n,m}$ and is actually a continuous function and not piecewise constant anymore we need to incorporate the projection \eqref{eq:projection_pc} after each convolution in \eqref{eq:approx_neural_network}. Given $\Omega_l=(a_l,b_l)\times (c_l,d_l)$, $\Sigma_l = (p_l,q_l)\times (r_l,s_l)$, and $h_l$, defining $N_l = \frac{c_l-d_l}{h_l}$, $M_l = \frac{b_l-a_l}{h_l}$, $n_l = \frac{s_l-r_l}{h_l}$, $m_l = \frac{p_l-q_l}{h_l}$ we find $\NN_h(u^h,w^h,b^h)(x^{h_L}_{n,m})=\NN_d(U,W,b^h)_{n,m}$ for all grid points $x_{n,m}^{h_L}$ if the output of the net has resolution $h_L$.

\subsection{Approximation Results} In order to relate \eqref{eq:approx_neural_network} and \eqref{eq:continuous_neural_network} we now consider the limit $h\rightarrow 0$. To this aim, we pose an intuitive and non-restrictive assumption on the down- and upsampling operators used in the CNN, which is satisfied, for instance, for downsampling with strides, max pooling, or average pooling and piecewise constant or bilinear upsampling.
\begin{assumption}\label{ass:sampling}
The down- and upsampling functions $f: \R^{s\times s}\rightarrow \R$ and $g: \R^{2\times 2}\rightarrow \R^{s\times s}$ satisfy 
\[\min_{i,j}U_{i,j}\leq f(U)\leq \max_{i,j}U_{i,j} \quad \text{for all } U\in \R^{s\times s}\]
\[ \min_{i,j}U_{i,j} \leq \min_{k,l} g(U)_{k,l} \leq \max_{k,l}g(U)_{k,l} \leq \max_{i,j}U_{i,j} \quad \text{for all } U\in \R^{2\times 2}.\]
\end{assumption}
Before proving convergence of \eqref{eq:approx_neural_network} for fine discretizations we provide several results about convergence of its individual components.

\begin{lemma}\label{lem:cont_padding}
Let $\theta$ and $\Omega$ be rectangular domains in $\R^2$ and $p \in [1,\infty]$. Then, if $u_n\rightarrow u$ in $C(\overline{\theta})$, then $\Rf_{\theta,\Omega}u_n\rightarrow \Rf_{\theta,\Omega}u$ in $C(\overline{\Omega})$ and if $u_n\rightarrow u$ in $L^p(\theta)$, then $\Rf_{\theta,\Omega}{u_n}\rightarrow \Rf_{\theta,\Omega}{u}$ in $L^p(\Omega)$.
\end{lemma}
\begin{proof}
This result is a simple consequence of the properties of the extension by reflection.
\end{proof}

\begin{lemma}{(Strong/weak to uniform continuity of the convolution)}\label{lem:continuity_convolution}
Let $p\in (1,\infty )$ and $u_n\rightarrow u$ in $L^p(\Omega_\Sigma)$ and $w_n\rightharpoonup w$ in $L^{p'}(\Sigma)$. Then $u_n*w_n \rightarrow u*w$ uniformly in $\Omega$.
\begin{proof}
First note that convergence of the sequence in $L^p$ follows from
\begin{equation*}
\begin{aligned}
\| u_n*w_n-u*w\|_p &\leq\| (u_n-u)*w_n \|_p + \| u*(w_n-w)\|_p\\
&\leq \|u_n-u\|_p \|w_n\|_{1} + \| \int_\Sigma u(.-y)*(w_n(y)-w(y))\;\wrt y\|_p\rightarrow 0,
\end{aligned}
\end{equation*}
where the first term in the last line goes to zero by convergence of $(u_n)_n$ and boundedness of $(w_n)_n$, and convergence to zero of the second term follows from Lebesgue's dominated convergence theorem using the weak convergence of $(w_n)_n$. We now extend this to uniform convergence using the Arzel\`{a}-Ascoli theorem, \cite[Satz 4.12]{alt2016linear} on $C(\overline{\Omega})$. In order to do so, we show uniform boundedness and uniform equicontinuity of $(u_n*w_n)_n$. The former is immediate since, by Hölder's inequality, $|u_n*w_n(x)| \leq \| u_n\|_p \|w_n\|_{p'}$ with the right-hand-side being bounded by convergence of the sequences. To show uniform equicontinuity, let $\epsilon>0$ be arbitrary, set $M\coloneqq \max_n \|w_n\|_{p'}$ and pick $n_0$ such that $\|u_{n}-u\|_p<\epsilon/(4M)$ for all $n>n_0$. Further choose $\delta>0$ such that, for $|h|<\delta$, $\|\tilde{u}(.+h)-\tilde{u}\|_p<\frac{\epsilon}{2M}$ and $\|\tilde{u}_n(.+h)-\tilde{u}_n\|_p<\frac{\epsilon}{M}$ for all $n\in\{0,1,\dots, n_0\}$ which is possible since the translation $v\mapsto v(.+h)$ is continuous from $L^p(\R^d) $ to $ L^p(\R^d)$. For $|h|<\delta$ we compute
\begin{equation*}
\begin{aligned}
|u_n*w_n(x+h)-u_n*w_n(x)| &\leq
\int_\Sigma |u_n(x+h-y)-u_n(x-y)|\;|w_n(y)|\; \wrt y \\
&\leq \|\tilde{u}_n(.+h)-\tilde{u}_n\|_p \|w_n\|_{p'} = (*)
\end{aligned}
\end{equation*}
which is upper bounded by $\epsilon$ in the case $n\leq n_0$. In the case $n>n_0$ we obtain
\begin{equation*}
\begin{aligned}
(*) \leq M (2\|u_n-u\|_p + \|\tilde{u}(.+h)-\tilde{u}\|_p) \leq \epsilon.
\end{aligned}
\end{equation*}
The Arzel\`{a}-Ascoli theorem now yields a subsequence $(u_{n_k}*w_{n_k})_k$ converging uniformly, and the limit has to be $u*f$ since uniform convergence implies $L^p$ convergence. Uniform convergence of the entire sequence can finally be obtained from uniqueness of the limit using a standard argument.
\end{proof}
\end{lemma}

\begin{lemma}\label{lem:Ph_bounded}
Let $\theta\subset\R^2$ be a bounded domain, $u^h, u\in C(\overline{\theta})$, and $u^h\rightarrow u$ uniformly as $h\rightarrow 0$. Then $P^h u^h\rightarrow u$ uniformly as $h\rightarrow 0$.
\end{lemma}
\begin{proof}
Using the triangle inequality, this follows from $\|P^h\| \leq 1$ uniformly in $h$ when regarded as operator on $C(\overline{\theta})$, and from $P^h u \rightarrow u$ uniformly.
\end{proof}

\begin{lemma}\label{lem:sampling_bounded}
Let $\theta\subset\R^2$ be a bounded domain, $u^h \in \PC^h(\theta)$ and $u\in C(\overline{\theta})$ such that $u^h\rightarrow u$ uniformly as $h\rightarrow 0$. If $f$ and $g$ satisfy \Cref{ass:sampling} then $\Df_s^h u^h\rightarrow u$ and $\Uf_s^h u^h\rightarrow u$ uniformly as $h\rightarrow 0$.
\end{lemma}
\begin{proof}
We consider only the case of downsampling, since the proof for upsampling is analogous. Let $x\in\Omega$ be arbitrary and for fixed $h$, let $i,j$ such that $x\in Q^h_{i,j}$.
\begin{equation}\label{eq:sampling_bounded}
\begin{aligned}
|u^h(x)-\Df_s^h u^h(x)| &= |u^h(x^h_{i,j})-\Df_s^h u^h(x^h_{i,j})| = |u^h(x^h_{i,j})-f( (u^h(x^h_{k,l}))_{k,l} )|\\
&\underbrace{\leq}_{(i)} \max_{k,l} |u^h(x^h_{i,j})- u^h(x^h_{k,l}))| \underbrace{\leq}_{(ii)} \max_{|y-z|^2\leq 2(s-1)^2h^2} |u^h(y)-u^h(z)|\\
&\leq 2\|u^h-u\|_\infty + \max_{|y-z|^2\leq 2(s-1)^2h^2} |u(y)-u(z)| \rightarrow 0 \text{ as } h\rightarrow 0.
\end{aligned}
\end{equation}
Above, $(k,l)$ runs over all neighboring nodes of $(i,j)$ that are used for the computation of $\Df_k^h u^h(x_{i,j})$. In $(i)$ we use \Cref{ass:sampling} and in $(ii)$ we estimate $\max_{k,l}|x_{i,j}-x_{k,l}|^2$. The uniform convergence to zero follows from the uniform convergence of $u^h$ and the uniform continuity of $u$. Using \eqref{eq:sampling_bounded} we can estimate
\begin{equation*}
\begin{aligned}
\|u-\Df_s^h u^h\|_\infty &\leq \|u-u^h\|_\infty + \|u^h-\Df_s^h u^h\|_\infty\rightarrow 0.
\end{aligned}
\end{equation*} 
\end{proof}

Collecting the previous results we can prove convergence of the CNN output as follows.

\begin{theorem}\label{thm:CNN_cont}
Let $1<p<\infty$, $u\in L^p(\Omega_0)$, $w\in \Wc^{p'}(\Sigma)$, and $b\in \R^L$. Let $u^h$ and $w^h$ be piecewise constant approximations of $u$ and $w$, respectively, and $b^h\rightarrow b$. Then $\NN_h(u^h,w^h,b^h)\rightarrow \NN(u,w,b)$ uniformly.
\begin{proof}
It is clear that $u^h \rightarrow u$ in $L^p$ and $w^h \rightarrow w$ in $L^{p'}$ as $h \rightarrow 0$. 
By \Cref{lem:cont_padding} and \Cref{lem:continuity_convolution}, we obtain $(\Rf_{\Omega_0,\Omega_0'}u^h) * w^{h,0}+b^{h,0} \rightarrow (\Rf_{\Omega_0,\Omega_0'}u) * w^{0}+b^{0}$ uniformly. Further, by Lipschitz continuity of $\sigma$, we find that $v^h\rightarrow v$ uniformly implies $\sigma(v^h)\rightarrow\sigma(v)$ uniformly. Using this, \Cref{lem:cont_padding}, \Cref{lem:continuity_convolution}, \Cref{lem:Ph_bounded}, \Cref{lem:sampling_bounded}, and continuity of the reflection padding we can pass the uniform convergence through the network to obtain the desired result.
\end{proof}
\end{theorem}
\begin{remark}  \label{rem:CNN_cont}
\Cref{thm:CNN_cont} even holds under weak convergence $w^h\rightharpoonup w$ as $h \rightarrow 0$ instead of $w^h \rightarrow w$. Weak convergence arises naturally in the case of boundedness of the $L^{p'}$ norm of the weights $(w^h)_h$ by reflexivity of $L^{p'}$ and the theorem of Banach-Alaoglu.
\end{remark}

\subsection{Network Training} In this section, we consider training a CNN by empirical risk minimization and show that, also in this setting, \eqref{eq:approx_neural_network} converges to \eqref{eq:continuous_neural_network} in a meaningful sense. Let us start by briefly recalling three central definitions in the context of optimization. A functional $F:X \rightarrow \R\cup \{\infty\}$ defined on any space $X$ is called \emph{proper} if it is not equal to $\infty$ everywhere. If $X$ is a normed space, we call $F$ \emph{coercive} if $F(x)\rightarrow\infty$ whenever $\|x\|\rightarrow \infty$. Lastly, we say that $F$ is \emph{(weakly) lower semicontinuous} if for any sequence $(x_n)_n\subset X$ converging (weakly) to $x\in X$ it holds true that $F(x)\leq\liminf_{n\rightarrow\infty}F(x_n)$.

Let $p\in (1,\infty)$, $Y$ be a Banach space, $(u_j,y_j)_{j=1}^N\subset L^p(\Omega_0)\times Y$ be pairs of network input and ground truth data or label, and $\ell: Y\times Y\rightarrow [0,\infty]$ a loss functional. Incorporating also a possible forward operator $A:L^\infty(\Omega)\rightarrow Y$ we consider the following minimization problems.
\begin{equation}
\tag{\text{$P_h$}}
\begin{aligned}
\min_{\substack{w^h\in \Wc^{p'}(\Sigma)\\ b^h\in\R^L}} \Ec^h(y^h,u^h, w^h,b^h)\coloneqq \sum_{j=1}^N \ell(y^h_j, A\NN_h(u^h_j,w^h,b^h)) + \lambda\|w^h\|_{p'} +\\ \nu |b^h| + \Ic_{\Wc_h(\Sigma)}(w^h)\label{eq:discrete_training}
\end{aligned}
\end{equation}
\begin{equation}
\tag{P}
\begin{aligned}
\min_{\substack{w\in \Wc^{p'}(\Sigma)\\ b\in\R^L}} \Ec(y,u,w,b)&\coloneqq \sum_{j=1}^N \ell(y_j, A\NN(u_j,w,b)) + \lambda\|w\|_{p'} + \nu|b|\label{eq:cont_training}.
\end{aligned}
\end{equation}
Here $|\;.\;|$ denotes any norm on $\R^L$. Above, \eqref{eq:discrete_training} is the training problem in the approximated, finite dimensional setting and \eqref{eq:cont_training} the infinite dimensional limit. The data $u^h_j$ shall be a piecewise constant approximation of $u_j$ such that $u^h_j\rightarrow u_j$. Analogously $y^h_j$ shall be a meaningful approximation of $y_j$ in the finite dimensional setting. We do, however, not specify any convergence $y^h_j\rightarrow y_j$ and take this into account in the notion of $\Gamma$-convergence, see \Cref{def:gamma_convergence} below. By choosing the operator $A$ and the data $(u_j,y_j)_j$ accordingly the above problems cover training a CNN as an end-to-end method as well as learning a data distribution or using the CNN as an image prior for inverse problems.

\begin{assumption}\label{ass:training} Let the data $(u^h_j,y^h_j)_{h,j}$ and the operator $A$ be such that:\
\begin{enumerate}
\item $A\in \Lc(L^\infty(\Omega),Y)$.
\item\label{ass:proper} For all $h,j$, $\ell(y^h_j,\;.\;): Y\rightarrow [0,\infty]$ is lower semicontinuous and there exists ($w^h,b^h)\in \Wc_h(\Sigma)\times \R^L$ such that for all $j$, $\ell(y^h_j,A\NN_h(u^h_j,w^h,b^h))<\infty$.
\item\label{data_term_convergence} For all $j$ and any $z^h\rightarrow z$ in $Y$ it holds that $\ell(y^h_j,z^h)\rightarrow \ell(y_j,z)$.
\end{enumerate}
\end{assumption}

\begin{ex}{(Example losses)}\label{ex:data_fid}\

\begin{itemize}
\item A popular functional satisfying \Cref{ass:training} is the cross entropy loss for classification, cf. \cite{ronneberger2015u}. In this case, $y\in Y:= [L^\infty(\Omega)]^n$ for $n$ different classes with $y_i(x)\in\{0,1\}$ the true labels of class $i$, and $\ell(y,\;.\;): [L^\infty(\Omega)]^n \rightarrow [0,\infty)$, is given as
\begin{equation}
\begin{aligned}
\ell(y,v) &= -\int_\Omega g(x)\sum_{i=1}^n \log(\sigma(v(x))_i)y_i(x) \wrt x
\end{aligned}
\end{equation}
with $g\in L^\infty(\Omega)$, $g\geq 0$ a weight map to give different importance to certain pixels in the image and $\sigma$ the softmax function.
This functional satisfies \Cref{ass:training} whenever $y^h_j\rightarrow y_j$ in $[L^p(\Omega)]^N$ for any $p\in [1,\infty]$ since $\ell :[L^\infty(\Omega)]^n \times [L^\infty(\Omega)]^n \rightarrow [0,\infty)$ is continuous with respect to $L^p$ convergence in the first an $L^\infty$ convergence in the second argument. 
\item Another common loss satisfying  \Cref{ass:training} is the q norm discrepancy, that is, for $1\leq q <\infty$, $Y=L^q(\Omega)$, $\ell(y,z)=\frac{1}{q}\|z-y\|_q^q$ with a sequence $\PC^h(\Omega)\ni y^h\rightarrow y$ in $L^q(\Omega)$.
\end{itemize}
\end{ex}

In the following our goal is to deduce existence of a solution of the infinite dimensional problem \eqref{eq:cont_training} based on the discrete problems \eqref{eq:discrete_training} for $h\rightarrow 0$. In order to do so we will make use of equicoercivity and (weak) $\Gamma$-convergence defined below in \Cref{def:gamma_convergence}. Careful readers will notice, that equicoercivity and (weak) $\Gamma$-convergence resemble the definitions of coercivity and lower semicontinuity, respectively, but for a sequence instead of a single function. As coercivity and weak lower semicontinuity are sufficient conditions for existence of a minimizer of a proper function on a reflexive Banach space, it is not surprising that equicoercivity and weak $\Gamma$-convergence are precisely the corresponding sufficient conditions for existence of a minimizer of the limit of a sequence of functions in the sense of \Cref{lem:gamma_convergence}.

\begin{definition}{(Equicoercivity and $\Gamma$-convergence, \cite{braides2002gamma})}\label{def:gamma_convergence}
Let $X$ be a Banach space and $F,F_h:X\rightarrow\overline{\R}$.
\begin{itemize}
\item We call $(F_h)_h$ equicoercive if there exists a coercive function $\overline F:X\rightarrow\overline{\R}$ such that $\overline F\leq F_h$ for all $h$. 
\item We say that $F_h$ weakly $\Gamma$-converges to $F$ as $h\rightarrow 0$ if for every $x\in X$ we have that for every sequence $x^h\rightharpoonup x$, $F(x)\leq \liminf_{h\rightarrow 0}F_h(x^h)$ and there exists a sequence $x^h\rightharpoonup x$ such that $F(x)\geq \limsup_{h\rightarrow 0}F_h(x^h).$
\end{itemize}
\end{definition}

\begin{lemma}\label{lem:gamma_convergence}
Let $X$ be reflexive, for each $h$, $F_h: X\rightarrow \overline{\R}$ be proper and such that $(F_h)_h$ is equicoercive and weakly $\Gamma$-converges to $F$ and $\hat{x}^h\in X$ be a minimizer of $F_h$. Then either
\begin{itemize}
\item $F_h(\hat{x}^h)\rightarrow \infty$ and $F\equiv\infty$, or
\item $F_h(\hat{x}^h)\rightarrow \min_{x\in X} F(x)<\infty$. 
\end{itemize}
In the latter case each weak accumulation point of $(\hat{x}_h)_h$ is a minimizer of $F$ and there exists at least one weak accumulation point. If $\hat{x}$ is the unique minimizer of $F$ the full sequence converges $\hat{x}^h\rightharpoonup \hat{x}$.
\end{lemma}
\begin{proof}
The proof is a direct modification of \cite[Theorem 1.21]{braides2002gamma}.
\end{proof}

\begin{lemma}\label{lem:func_conv}
Let \Cref{ass:training} hold and assume that for all $j$, $u^h_j\rightarrow u_j$ in $L^p(\Omega_0)$. Then $F_h\coloneqq \Ec^h(y^h,u^h,\;.\;,\;.)$ is equicoercive and weakly $\Gamma$-converges to $F\coloneqq \Ec(y,u,\;.\;,\;.)$.
\end{lemma}
\begin{proof}
The equicoercivity follows from $ \lambda\|w\|_{p'} + \nu |b| \leq F_h(w,b)$.
To show the $\liminf$ inequality for weak $\Gamma$-convergence, take $w^h\rightharpoonup w$ and $b^h\rightarrow b$. Then, due to the strong convergence of $u^h_j$, by \Cref{thm:CNN_cont} and \Cref{rem:CNN_cont} we find that $\NN_h(u^h_j,w^h,b^h)\rightarrow \NN (u_j,w,b)$ uniformly. Since $A$ is bounded, also $A\NN_h(u^h_j,w^h,b^h)\rightarrow A\NN (u_j,w,b)$. Using \Cref{ass:training}, \eqref{data_term_convergence}, the weak lower semicontinuity of $\|\; . \;\|_{p'}$ and the fact that we always have $0\leq \liminf_{h\rightarrow 0}\Ic_{\Wc_h(\Sigma)}(w^h)$, we conclude
\begin{multline*}
\sum_j \ell(y_j,A\NN(u_j,w,b)) + \lambda\|w\|_{p'} +\nu |b| \leq\\
\liminf_{h\rightarrow 0} \left[ \sum_j\ell(y^h_j,A\NN_h(u^h_j,w^h,b^h)) + \lambda\|w^h\|_{p'} +\nu |b^h| + \Ic_{\Wc_h(\Sigma)}(w^h)\right].
\end{multline*}
To prove the $\limsup$ inequality, let $w^h \in \Wc_h(\Sigma)$ such that $w^h\rightarrow w$ in $L^{p'}$, which exists by density, and $b^h=b\in\R^L$ for all $h$. As above we find that $A\NN_h(u^h_j,w^h,b^h)\rightarrow A\NN (u_j,w,b)$. Therefore, using also the continuity of $\|\; . \;\|_{p'}$ wrt. to strong convergence, we can compute
\begin{multline*}
\limsup_{h\rightarrow 0} \left[\sum_j \ell(y^h_j,A\NN_h(u^h_j,w^h,b^h)) + \lambda\|w^h\|_{p'}+\nu |b^h| + \underbrace{\Ic_{\Wc_h(\Sigma)}(w^h)}_{=0}\right]  \\
= \lim_{h\rightarrow 0} \left[\sum_j \ell(y^h_j,A\NN_h(u^h_j,w^h,b^h)) + \lambda\|w^h\|_{p'} +\nu|b^h| \right]  \\
= \sum_j\ell(y_j,A\NN(u_j,w,b)) + \lambda\|w\|_{p'}+\nu|b|
\end{multline*}
\end{proof}

\begin{theorem}\label{thm:convergence}
Let \Cref{ass:training} hold, and let $u^h\rightarrow u$ in $L^p(\Omega_0)$. Then for each $h$, $\Ec^h(y^h,u^h,\;.\;,\;.)$ admits a minimizer $(\hat{w}^h,\hat{b}^h)$ and either
\begin{itemize}
\item $\lim_{h\rightarrow 0} \Ec^h(y^h,u^h, \hat{w}^h,\hat{b}^h) = \infty$ and $\Ec(y,u,\;.\;,\;.)$ is not proper, or
\item $\lim_{h\rightarrow 0} \Ec^h(y^h,u^h, \hat{w}^h,\hat{b}^h) = \min\limits_{\substack{w\in \Wc^{p'}(\Sigma)\\b\in\R^L}} \Ec(y,u, w,b)$.
\end{itemize}
In the latter case each weak accumulation point of $(\hat{w}^h,\hat{b}^h)_h$ solves $\min_{w,b} \Ec(y,u, w,b)$ and there exists at least one weak accumulation point.
\begin{proof}
Note that by \Cref{ass:training}, $\min_{w^h,b^h}\Ec^h(y^h,u^h,w^h,b^h)$ is a minimization problem of a proper, coercive, and weakly lower semicontinuous functional over a reflexive Banach implying existence of a minimizer. We can conclude using $\Gamma$-convergence of $\Ec^h(y^h,u^h,\;.\;,\;.)$ and \Cref{lem:gamma_convergence}.
\end{proof}
\end{theorem}

\Cref{thm:convergence} justifies making claims about the discrete setting based on results in the infinite dimensional setting. It states that, if the objective functional $\Ec(y,u,\;.\;,\;.)$ is proper, the solutions of the discrete training problems indeed approximate the solution of the continuous training problem and, therefore, any regularity guaranteed for the output $\NN(u,w,b)$ in the infinite dimensional setting will be reflected in the discrete approximation $\NN_h(u^h,w^h,b^h)$ given the resolution $h>0$ is sufficiently small. The condition of $\Ec(y,u,\;.\;,\;.)$ being proper is not restrictive, as properness solely requires that there is at least one point, where the functional is not equal to infinity. Considering the definition of the objective functional at hand in \eqref{eq:cont_training}, issues may arise at most from the loss $\ell$ as the other terms are finite everywhere. While both losses presented in \Cref{ex:data_fid} are, indeed, proper a more careful treatment would be necessary if, e.g., an indicator function is used as a loss as sometimes seen in inpainting experiments \cite[Section 3.1]{habring2022generative}. In this case, properness of the objective requires the used CNN to be capable of fitting the data exactly in the known data points. As a remedy -- and also to avoid non-differentiability of the indicator function -- an $L^2$ loss on the known data points is frequently used as a surrogate loss function \cite[Section 3, Inpainting]{deep_image_prior} for which properness is not an issue.

\section{Numerical Experiments}\label{sec:numerical_results}
For details regarding the implementation we refer to \Cref{sec:appendix_implementation} and to the publicly available source code \cite{cnn_regularity_git}.
\paragraph{End-to-end imaging} For end to end imaging applications, where the network input is either a corrupted image for restoration \cite{mao2016image,mohan2019robust}, or a clean image for segmentation \cite{ronneberger2015u}, our results predict a continuous or continuously differentiable output, respectively, as long as the convolution kernels admit $L^2$ regularity.
Experimental results supporting this claim can be found in \Cref{fig:unet_experiments_wd,fig:unet_experiments_resolution}. We perform end-to-end denoising from additive Gaussian noise with zero mean and a standard deviation of 0.07 with a U-net. In \Cref{fig:unet_experiments_wd} we show reconstructions for the same network trained with different regularization parameters. Precisely, we trained the network with increasing values of the weight decay of Pytorch's ADAM optimizer, which corresponds to the parameter of the $\ell_2$ regularization of the network weights. One can observe that, as the weight decay increases, the reconstruction becomes progressively smoother and tends to decrease the magnitude of jumps and in some cases eliminate them at all, which also decreases contrast in the image. This can be observed best, when inspecting the 1D plots of the cuts through the images. There, kinks are significantly smoothed and jumps are filled as the weight decay increases. In \Cref{fig:unet_experiments_resolution} we compare results for different resolutions. We used the network from \Cref{fig:unet_experiments_wd} trained with a weight decay of $0.001$ and first doubled, then tripled the resolution by doubling/tripling the sizes of all filter kernels. For training and testing we used the same images sampled at the corresponding resolutions. The results indicate, indeed, convergence to a continuous output in the infinite resolution limit. In particular, one can observe again that jumps in the output images are filled as the resolution increases.

\renewcommand\w{1.9cm}
\newcommand\hw{0.55cm}
\newcommand\h{1cm}

\newcommand\im{0}
\newcommand\scaling{1}
\newcommand\nepochs{40}

\begin{figure}
\centering
\begin{subfigure}{\textwidth}
\centering
\begin{subfigure}{\w}
\caption*{Noisy}
\includegraphics[width=\w]{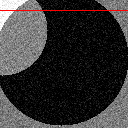}%
\end{subfigure}%
\begin{subfigure}{\w}
\caption*{GT}
\includegraphics[width=\w]{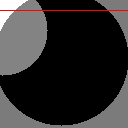}%
\end{subfigure}%
\begin{subfigure}{\w}
\caption*{WD 0}
\includegraphics[width=\w]{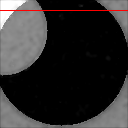}%
\end{subfigure}%
\begin{subfigure}{\w}
\caption*{WD 0.001}
\includegraphics[width=\w]{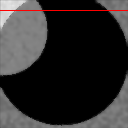}%
\end{subfigure}%
\begin{subfigure}{\w}
\caption*{WD 0.01}
\includegraphics[width=\w]{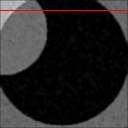}%
\end{subfigure}%
\begin{subfigure}{\w}
\caption*{WD 0.1}
\includegraphics[width=\w]{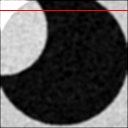}%
\end{subfigure}%
\begin{subfigure}{\w}
\caption*{WD 0.5}
\includegraphics[width=\w]{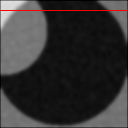}%
\end{subfigure}%
\end{subfigure}
\begin{subfigure}{\textwidth}
\centering
\includegraphics[width=\w]{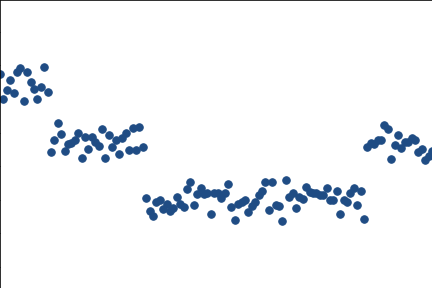}%
\includegraphics[width=\w]{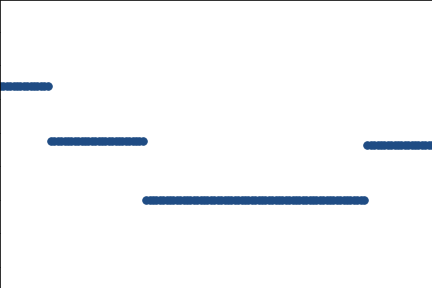}%
\includegraphics[width=\w]{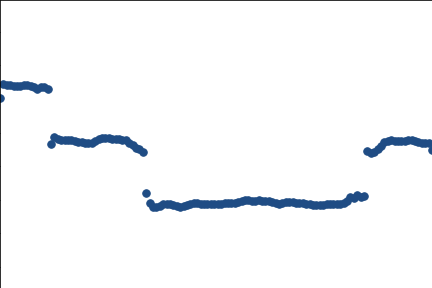}%
\includegraphics[width=\w]{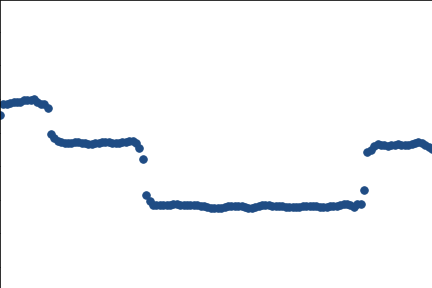}%
\includegraphics[width=\w]{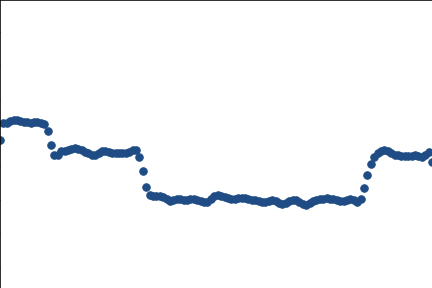}%
\includegraphics[width=\w]{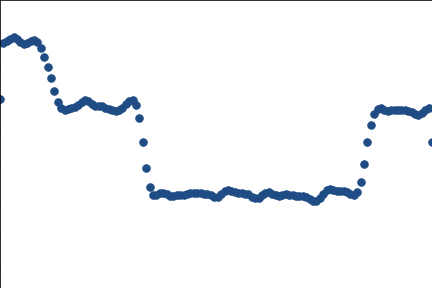}%
\includegraphics[width=\w]{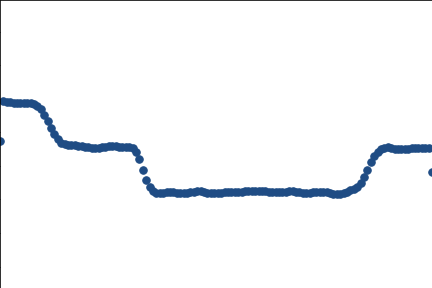}%
\end{subfigure}

\renewcommand\im{1}
\begin{subfigure}{\textwidth}
\centering
\includegraphics[width=\w]{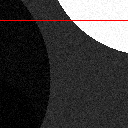}%
\includegraphics[width=\w]{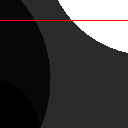}%
\includegraphics[width=\w]{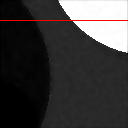}%
\includegraphics[width=\w]{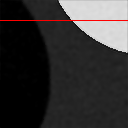}%
\includegraphics[width=\w]{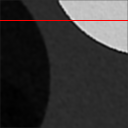}%
\includegraphics[width=\w]{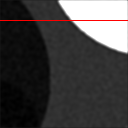}%
\includegraphics[width=\w]{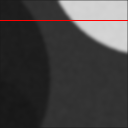}%
\end{subfigure}
\begin{subfigure}{\textwidth}
\centering
\includegraphics[width=\w]{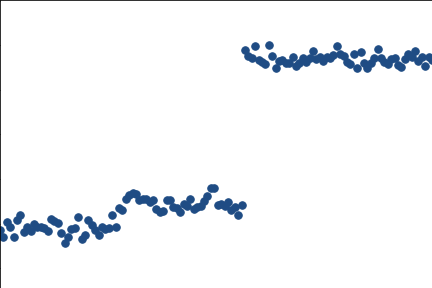}%
\includegraphics[width=\w]{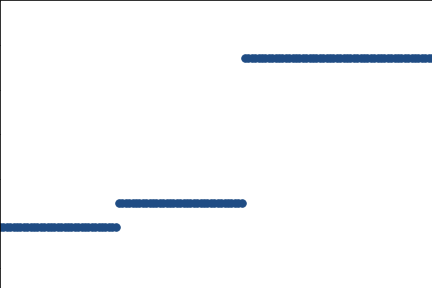}%
\includegraphics[width=\w]{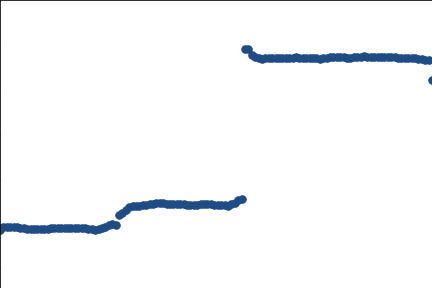}%
\includegraphics[width=\w]{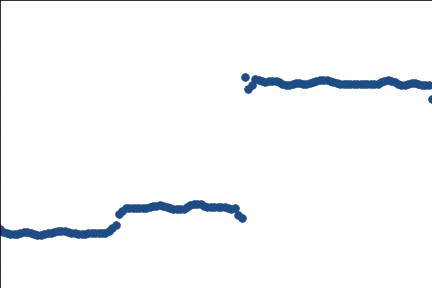}%
\includegraphics[width=\w]{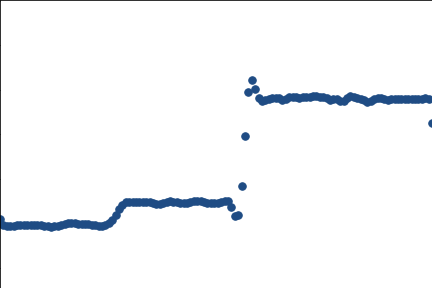}%
\includegraphics[width=\w]{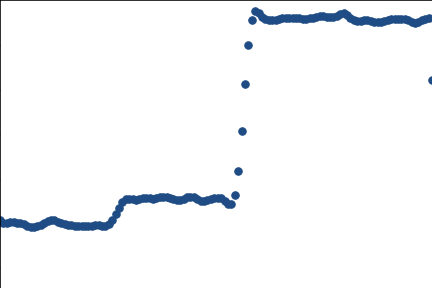}%
\includegraphics[width=\w]{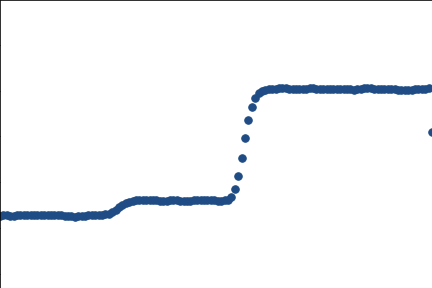}%
\end{subfigure}

\renewcommand\im{2}
\begin{subfigure}{\textwidth}
\centering
\includegraphics[width=\w]{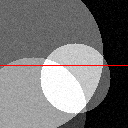}%
\includegraphics[width=\w]{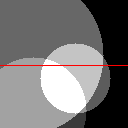}%
\includegraphics[width=\w]{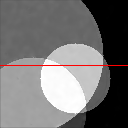}%
\includegraphics[width=\w]{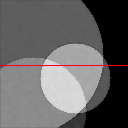}%
\includegraphics[width=\w]{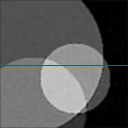}%
\includegraphics[width=\w]{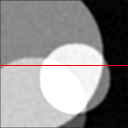}%
\includegraphics[width=\w]{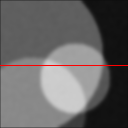}%
\end{subfigure}
\begin{subfigure}{\textwidth}
\centering
\includegraphics[width=\w]{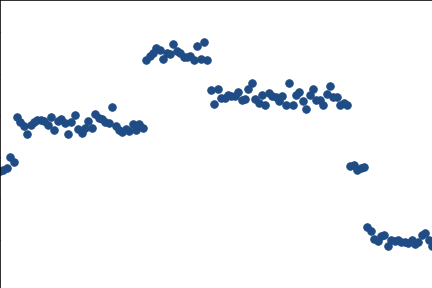}%
\includegraphics[width=\w]{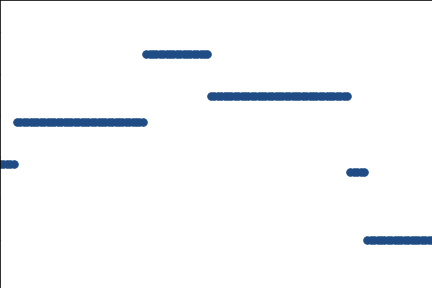}%
\includegraphics[width=\w]{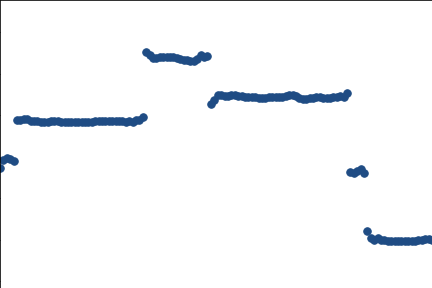}%
\includegraphics[width=\w]{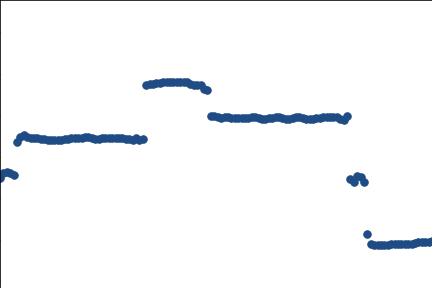}%
\includegraphics[width=\w]{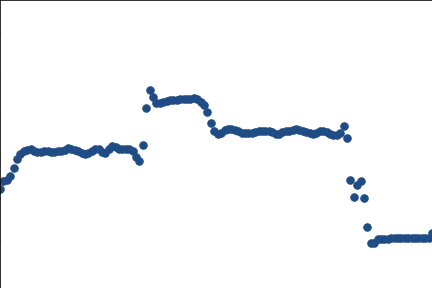}%
\includegraphics[width=\w]{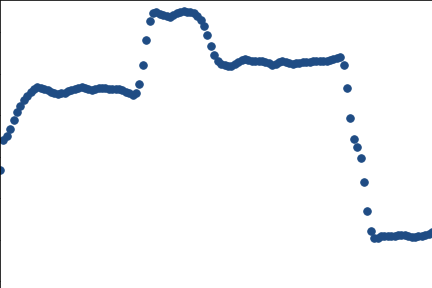}%
\includegraphics[width=\w]{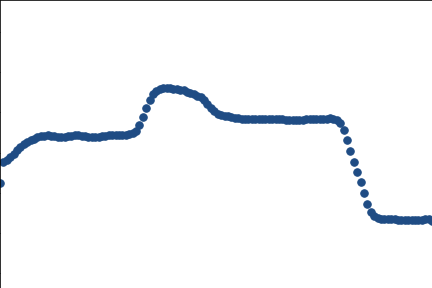}%
\end{subfigure}

\renewcommand\im{3}
\begin{subfigure}{\textwidth}
\centering
\includegraphics[width=\w]{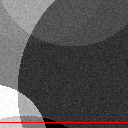}%
\includegraphics[width=\w]{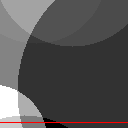}%
\includegraphics[width=\w]{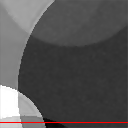}%
\includegraphics[width=\w]{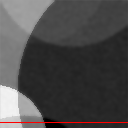}%
\includegraphics[width=\w]{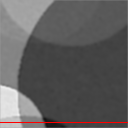}%
\includegraphics[width=\w]{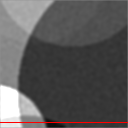}%
\includegraphics[width=\w]{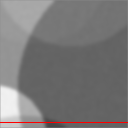}%
\end{subfigure}
\begin{subfigure}{\textwidth}
\centering
\includegraphics[width=\w]{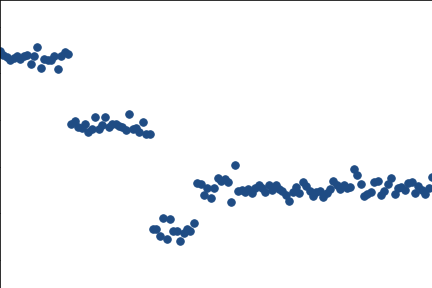}%
\includegraphics[width=\w]{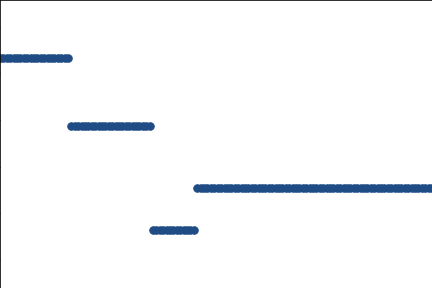}%
\includegraphics[width=\w]{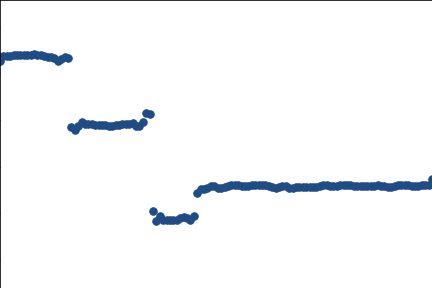}%
\includegraphics[width=\w]{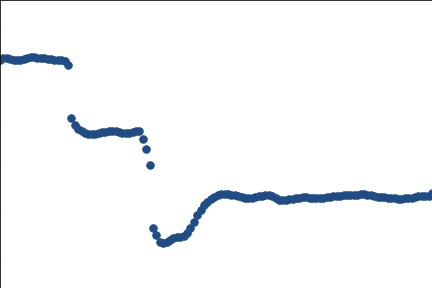}%
\includegraphics[width=\w]{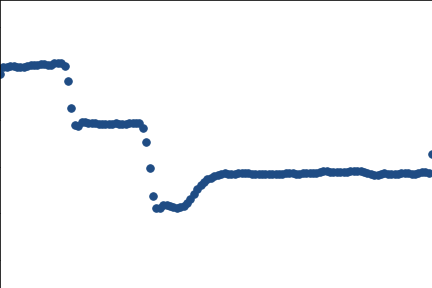}%
\includegraphics[width=\w]{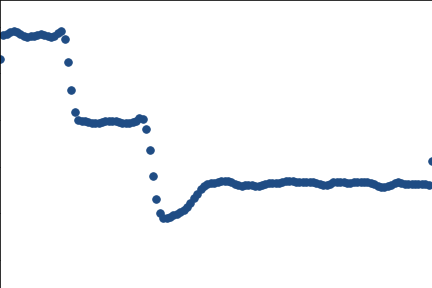}%
\includegraphics[width=\w]{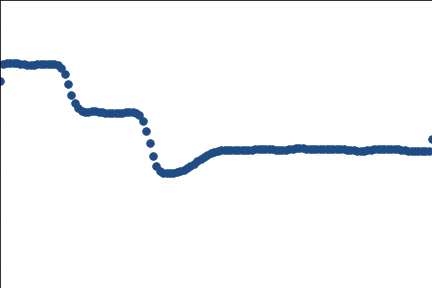}%
\end{subfigure}

\renewcommand\im{4}
\begin{subfigure}{\textwidth}
\centering
\includegraphics[width=\w]{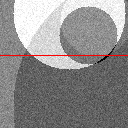}%
\includegraphics[width=\w]{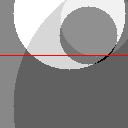}%
\includegraphics[width=\w]{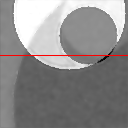}%
\includegraphics[width=\w]{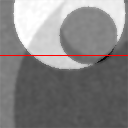}%
\includegraphics[width=\w]{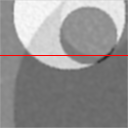}%
\includegraphics[width=\w]{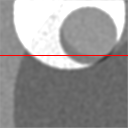}%
\includegraphics[width=\w]{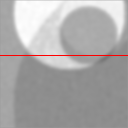}%
\end{subfigure}
\begin{subfigure}{\textwidth}
\centering
\includegraphics[width=\w]{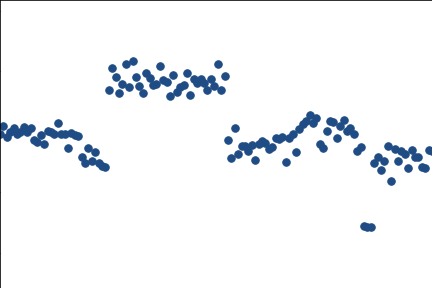}%
\includegraphics[width=\w]{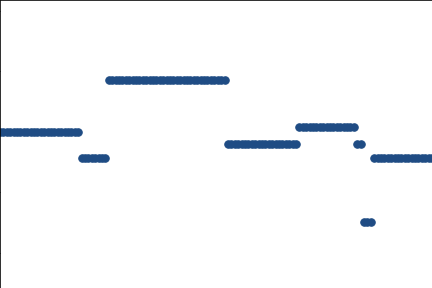}%
\includegraphics[width=\w]{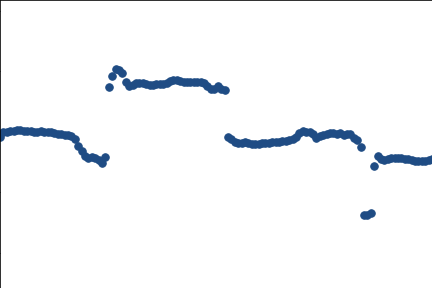}%
\includegraphics[width=\w]{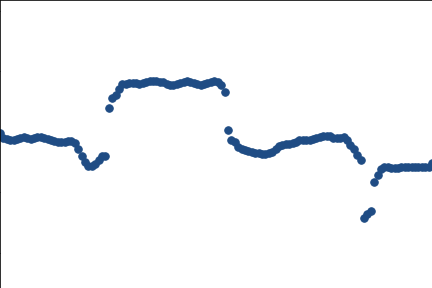}%
\includegraphics[width=\w]{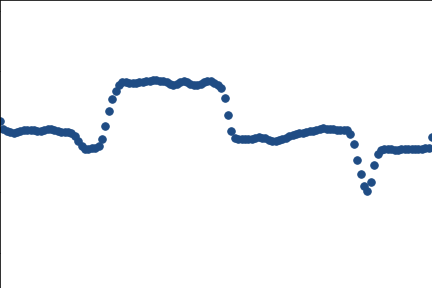}%
\includegraphics[width=\w]{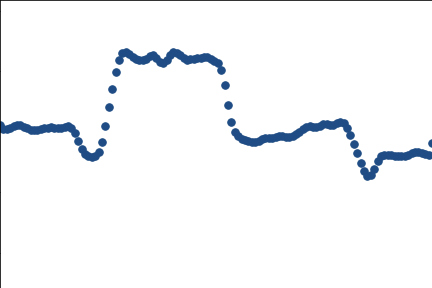}%
\includegraphics[width=\w]{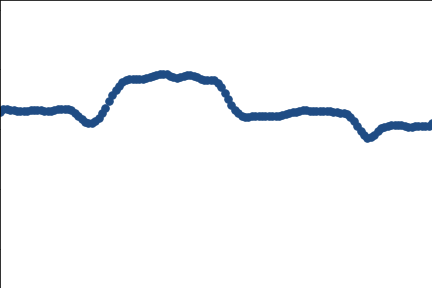}%
\end{subfigure}

\caption{Denoising end-to-end with a U-net. From left to right: Noisy data, ground truth data, reconstruction for a network trained with weight decay WD 0, 0.001, 0.01, 0.1, 0.5. Odd rows show images, even rows show a 1D plot of the red cut through the images above.}
\label{fig:unet_experiments_wd}
\end{figure}

\renewcommand\w{2.25cm}
\renewcommand\hw{0.55cm}
\renewcommand\h{1cm}

\renewcommand\im{0}
\renewcommand\scaling{1}
\newcommand\wed{0,01}

\newcommand\wdi{0,001}
\newcommand\wdii{0,00025}
\newcommand\wdiii{0,00011111111111111112}

\begin{figure}
\centering
\begin{subfigure}{\textwidth}
\centering
\begin{subfigure}[B]{\w}
\caption*{Noisy}
\includegraphics[width=\w]{codes_data_results_unet_image_0_noisy.png}%
\end{subfigure}%
\begin{subfigure}[B]{\w}
\caption*{GT}
\includegraphics[width=\w]{codes_data_results_unet_image_0_ground.png}%
\end{subfigure}%
\begin{subfigure}[B]{\w}
\caption*{Baseline\\resolution}
\includegraphics[width=\w]{codes_data_results_unet_size_orig_128_scaling_1_weight_decay_0,001_n_epochs_40_image_0.png}%
\end{subfigure}%
\begin{subfigure}[B]{\w}
\caption*{Double\\resolution}
\includegraphics[width=\w]{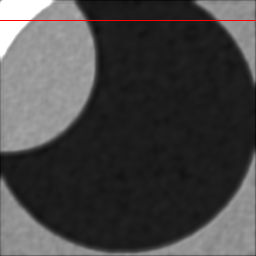}%
\end{subfigure}%
\begin{subfigure}[B]{\w}
\caption*{Triple\\resolution}
\includegraphics[width=\w]{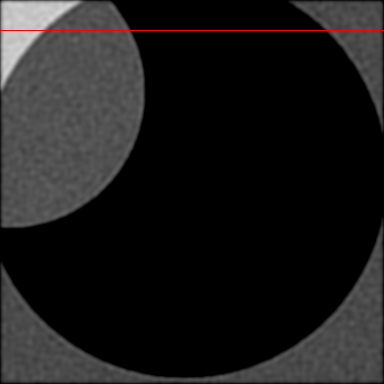}%
\end{subfigure}%
\end{subfigure}
\begin{subfigure}{\textwidth}
\centering
\includegraphics[width=\w]{codes_data_results_unet_image_0_noisy_cut.png}%
\includegraphics[width=\w]{codes_data_results_unet_image_0_ground_cut.png}%
\includegraphics[width=\w]{codes_data_results_unet_size_orig_128_scaling_1_weight_decay_0,001_n_epochs_40_image_0_cut.png}%
\includegraphics[width=\w]{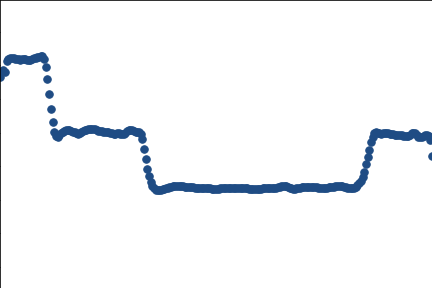}%
\includegraphics[width=\w]{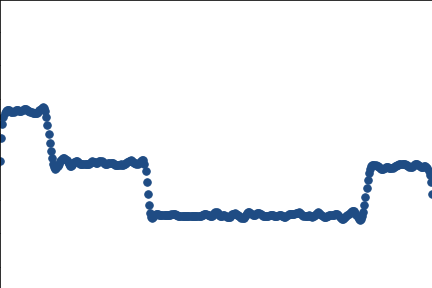}%
\end{subfigure}

\renewcommand\im{1}
\begin{subfigure}{\textwidth}
\centering
\begin{subfigure}{\w}
\includegraphics[width=\w]{codes_data_results_unet_image_1_noisy.png}%
\end{subfigure}%
\begin{subfigure}{\w}
\includegraphics[width=\w]{codes_data_results_unet_image_1_ground.png}%
\end{subfigure}%
\begin{subfigure}{\w}
\includegraphics[width=\w]{codes_data_results_unet_size_orig_128_scaling_1_weight_decay_0,001_n_epochs_40_image_1.png}%
\end{subfigure}%
\begin{subfigure}{\w}
\includegraphics[width=\w]{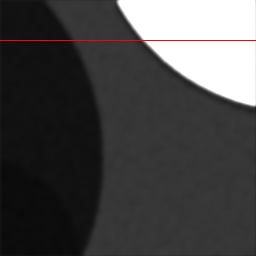}%
\end{subfigure}%
\begin{subfigure}{\w}
\includegraphics[width=\w]{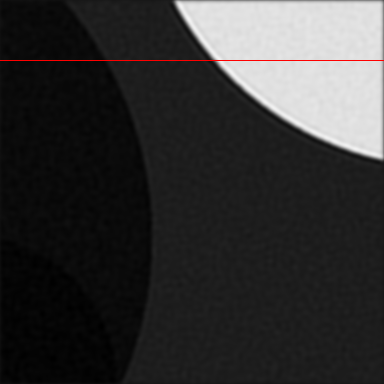}%
\end{subfigure}%
\end{subfigure}
\begin{subfigure}{\textwidth}
\centering
\includegraphics[width=\w]{codes_data_results_unet_image_1_noisy_cut.png}%
\includegraphics[width=\w]{codes_data_results_unet_image_1_ground_cut.png}%
\includegraphics[width=\w]{codes_data_results_unet_size_orig_128_scaling_1_weight_decay_0,001_n_epochs_40_image_1_cut.png}%
\includegraphics[width=\w]{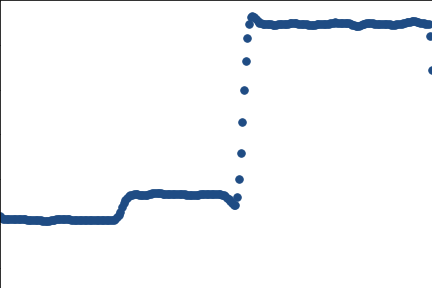}%
\includegraphics[width=\w]{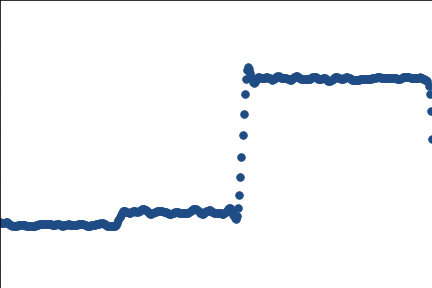}%
\end{subfigure}

\renewcommand\im{2}
\begin{subfigure}{\textwidth}
\centering
\begin{subfigure}{\w}
\includegraphics[width=\w]{codes_data_results_unet_image_2_noisy.png}%
\end{subfigure}%
\begin{subfigure}{\w}
\includegraphics[width=\w]{codes_data_results_unet_image_2_ground.png}%
\end{subfigure}%
\begin{subfigure}{\w}
\includegraphics[width=\w]{codes_data_results_unet_size_orig_128_scaling_1_weight_decay_0,001_n_epochs_40_image_2.png}%
\end{subfigure}%
\begin{subfigure}{\w}
\includegraphics[width=\w]{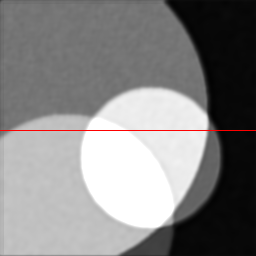}%
\end{subfigure}%
\begin{subfigure}{\w}
\includegraphics[width=\w]{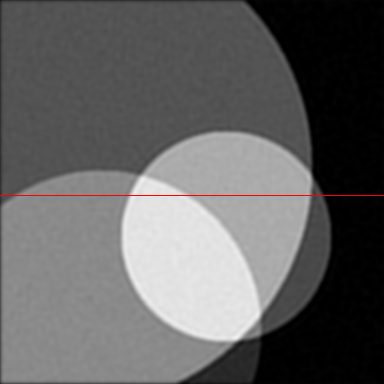}%
\end{subfigure}%
\end{subfigure}
\begin{subfigure}{\textwidth}
\centering
\includegraphics[width=\w]{codes_data_results_unet_image_2_noisy_cut.png}%
\includegraphics[width=\w]{codes_data_results_unet_image_2_ground_cut.png}%
\includegraphics[width=\w]{codes_data_results_unet_size_orig_128_scaling_1_weight_decay_0,001_n_epochs_40_image_2_cut.png}%
\includegraphics[width=\w]{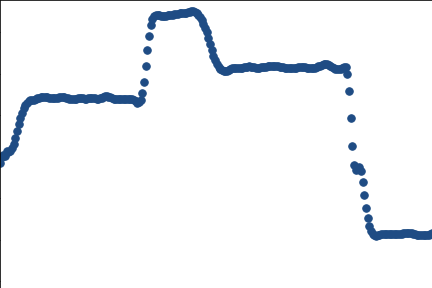}%
\includegraphics[width=\w]{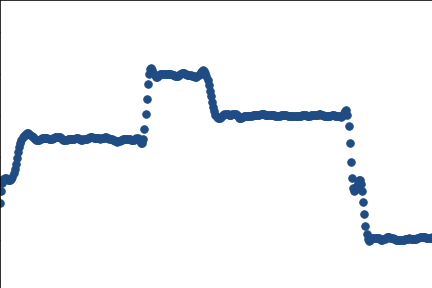}%
\end{subfigure}

\renewcommand\im{3}
\begin{subfigure}{\textwidth}
\centering
\begin{subfigure}{\w}
\includegraphics[width=\w]{codes_data_results_unet_image_3_noisy.png}%
\end{subfigure}%
\begin{subfigure}{\w}
\includegraphics[width=\w]{codes_data_results_unet_image_3_ground.png}%
\end{subfigure}%
\begin{subfigure}{\w}
\includegraphics[width=\w]{codes_data_results_unet_size_orig_128_scaling_1_weight_decay_0,001_n_epochs_40_image_3.png}%
\end{subfigure}%
\begin{subfigure}{\w}
\includegraphics[width=\w]{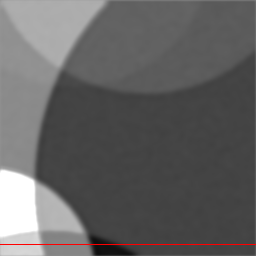}%
\end{subfigure}%
\begin{subfigure}{\w}
\includegraphics[width=\w]{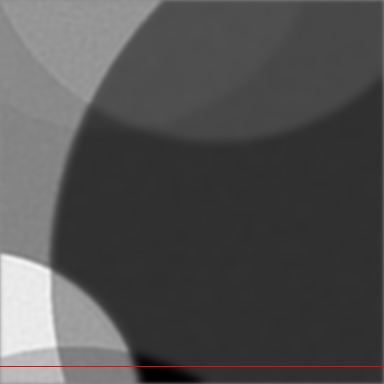}%
\end{subfigure}%
\end{subfigure}
\begin{subfigure}{\textwidth}
\centering
\includegraphics[width=\w]{codes_data_results_unet_image_3_noisy_cut.png}%
\includegraphics[width=\w]{codes_data_results_unet_image_3_ground_cut.png}%
\includegraphics[width=\w]{codes_data_results_unet_size_orig_128_scaling_1_weight_decay_0,001_n_epochs_40_image_3_cut.png}%
\includegraphics[width=\w]{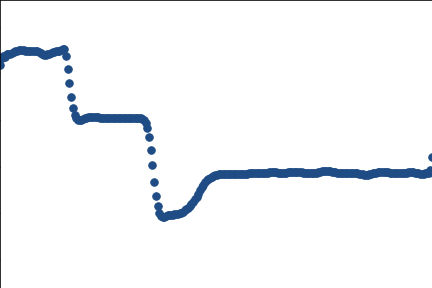}%
\includegraphics[width=\w]{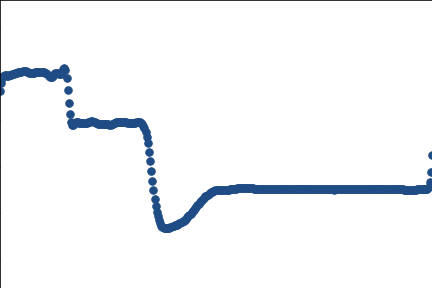}%
\end{subfigure}

\renewcommand\im{4}
\begin{subfigure}{\textwidth}
\centering
\begin{subfigure}{\w}
\includegraphics[width=\w]{codes_data_results_unet_image_4_noisy.png}%
\end{subfigure}%
\begin{subfigure}{\w}
\includegraphics[width=\w]{codes_data_results_unet_image_4_ground.png}%
\end{subfigure}%
\begin{subfigure}{\w}
\includegraphics[width=\w]{codes_data_results_unet_size_orig_128_scaling_1_weight_decay_0,001_n_epochs_40_image_4.png}%
\end{subfigure}%
\begin{subfigure}{\w}
\includegraphics[width=\w]{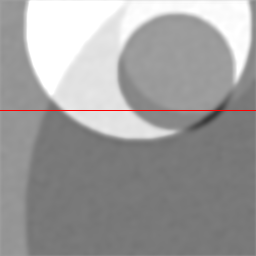}%
\end{subfigure}%
\begin{subfigure}{\w}
\includegraphics[width=\w]{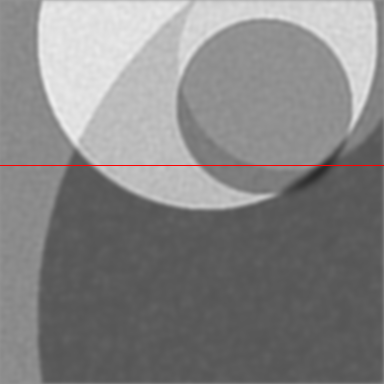}%
\end{subfigure}%
\end{subfigure}
\begin{subfigure}{\textwidth}
\centering
\includegraphics[width=\w]{codes_data_results_unet_image_4_noisy_cut.png}%
\includegraphics[width=\w]{codes_data_results_unet_image_4_ground_cut.png}%
\includegraphics[width=\w]{codes_data_results_unet_size_orig_128_scaling_1_weight_decay_0,001_n_epochs_40_image_4_cut.png}%
\includegraphics[width=\w]{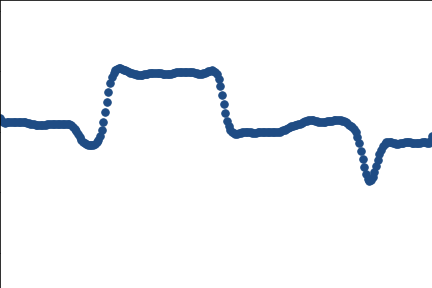}%
\includegraphics[width=\w]{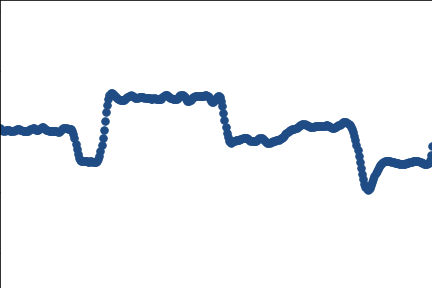}%
\end{subfigure}

\caption{Denoising end-to-end with a U-net. From left to right: Noisy data, ground truth data, reconstruction with the \emph{baseline} resolution and weight decay 0.001, then double resolution, and triple resolution. Odd rows show images, even rows show a 1D plot of the red cut through the images above.}
\label{fig:unet_experiments_resolution}
\end{figure}

\paragraph{Deep image prior} The deep image prior (DIP) \cite{deep_image_prior} is an untrained generative CNN used for solving general inverse problems in imaging. The solution is obtained as the output of a CNN with a fixed input (either uniformly distributed random noise or a smooth meshgrid) and parameters that are fitted to the corrupted data combined with early stopping. 
For this setting our analysis predicts the network output in the infinite resolution limit to be continuous for the random input and $C^1$ for the meshgrid input, as long as the $L^2$ norm of the weights is penalized (or remains bounded) during the optimization. We apply the DIP to denoising. As corruption we use additive Gaussian noise with mean 0, STD 0.1 times the image range and afterwards clip the image to $[0,1]$. Hyperparameters are set as proposed in \cite{deep_image_prior} for denoising the F16 image (for details see \Cref{appendix:dip}). The experimental results shown in \Cref{fig:dip_experiments} confirm our theoretical prediction: In early iterations the DIP output is significantly smoother than a reference result obtained with TV regularization. This is apparent in a blurrier image as well as in the 1D plots of the cut through the image, where again one can observe that the DIP results tend to smoothen jump discontinuities. At later iterations, also discontinuities and noise are reconstructed, while the $L^2$ norm of the network weights progressively increases as shown in \Cref{fig:dip_graphs}. The latter can be interpreted as the weights losing $L^2$ regularity, such that continuity of the network output can no longer be guaranteed. These observations hold for both the random and the meshgrid input, where for the latter, as predicted by our analysis, the regularity of the output is higher. For the experiments we used the publicly available code for the deep image prior \cite{deep_image_prior_git} licensed under Apache License 2.0.

\renewcommand\w{2cm}
\renewcommand\hw{0.6cm}
\renewcommand\hw{2cm}
\renewcommand\h{1cm}

\newcommand\numit{600}
\newcommand\NUMIT{59900}

\begin{figure}
\centering
\begin{subfigure}{\textwidth}
\centering
\begin{subfigure}[B]{\w}
\caption*{Noisy}
\includegraphics[width=\w]{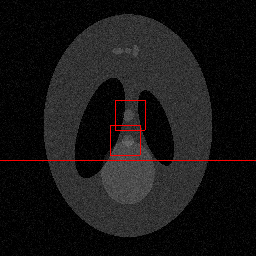}
\includegraphics[width=\hw]{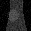}
\includegraphics[width=\hw]{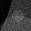}
\includegraphics[width=\w]{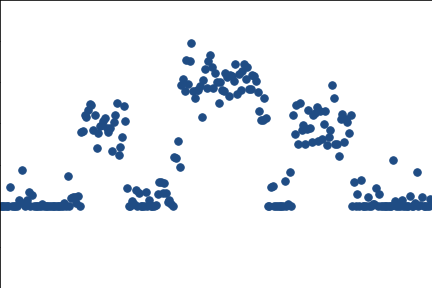}%
\end{subfigure}%
\begin{subfigure}[B]{\w}
\caption*{GT}
\includegraphics[width=\w]{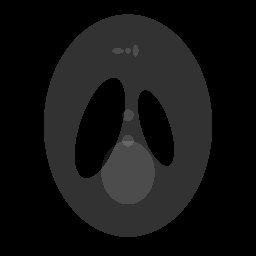}
\includegraphics[width=\hw]{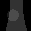}
\includegraphics[width=\hw]{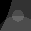}
\includegraphics[width=\w]{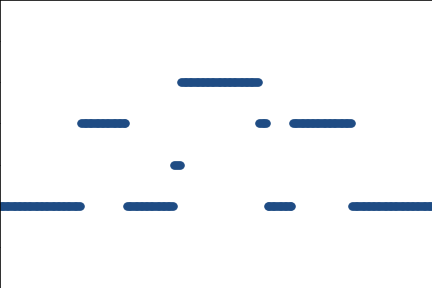}%
\end{subfigure}%
\begin{subfigure}[B]{\w}
\caption*{TV}
\includegraphics[width=\w]{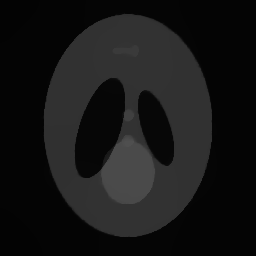}
\includegraphics[width=\hw]{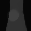}
\includegraphics[width=\hw]{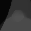}
\includegraphics[width=\w]{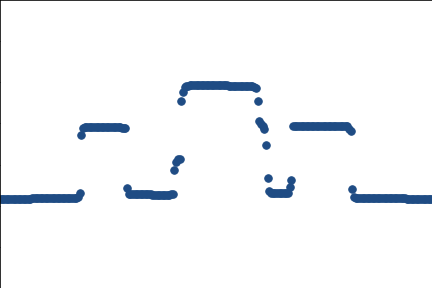}%
\end{subfigure}%
\begin{subfigure}[B]{\w}
\caption*{300}
\includegraphics[width=\w]{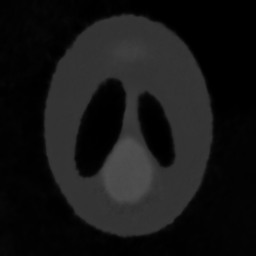}
\includegraphics[width=\hw]{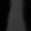}
\includegraphics[width=\hw]{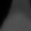}
\includegraphics[width=\w]{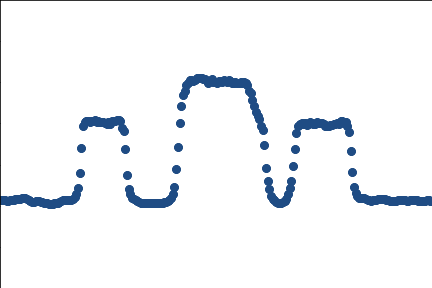}%
\end{subfigure}%
\begin{subfigure}[B]{\w}
\caption*{1700}
\includegraphics[width=\w]{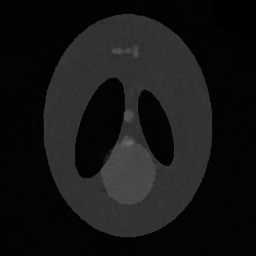}
\includegraphics[width=\hw]{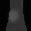}
\includegraphics[width=\hw]{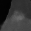}
\includegraphics[width=\w]{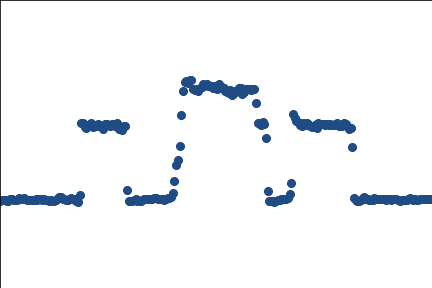}%
\end{subfigure}%
\begin{subfigure}[B]{\w}
\caption*{4900}
\includegraphics[width=\w]{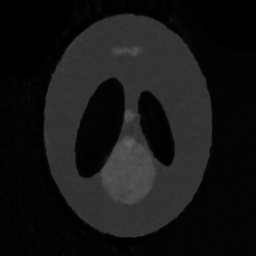}
\includegraphics[width=\hw]{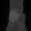}
\includegraphics[width=\hw]{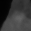}
\includegraphics[width=\w]{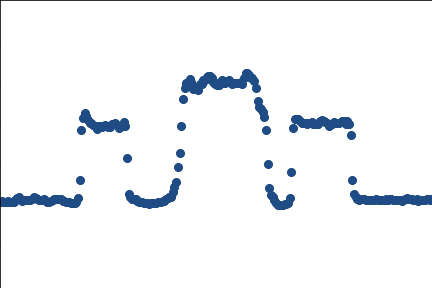}%
\end{subfigure}%
\begin{subfigure}[B]{\w}
\caption*{20000\\iterations}
\includegraphics[width=\w]{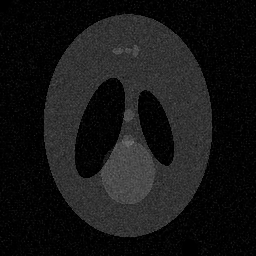}
\includegraphics[width=\hw]{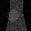}
\includegraphics[width=\hw]{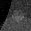}
\includegraphics[width=\w]{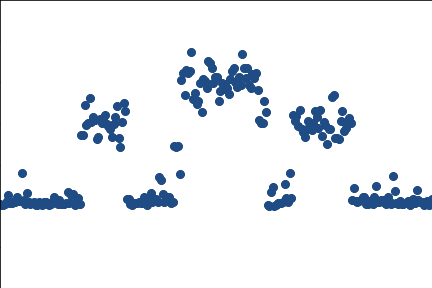}%
\end{subfigure}%
\caption{Noise input}
\end{subfigure}%

\centering
\begin{subfigure}{\textwidth}
\centering
\begin{subfigure}{\w}
\includegraphics[width=\w]{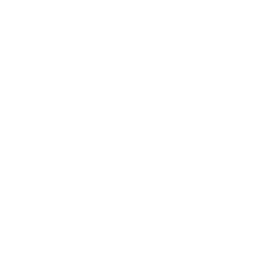}
\end{subfigure}%
\begin{subfigure}{\w}
\includegraphics[width=\w]{codes_data_results_white.png}
\end{subfigure}%
\begin{subfigure}{\w}
\includegraphics[width=\w]{codes_data_results_white.png}
\end{subfigure}%
\begin{subfigure}{\w}
\caption*{1000}
\includegraphics[width=\w]{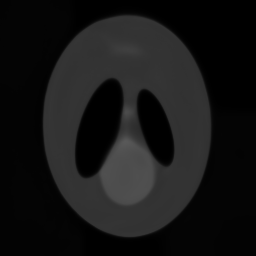}
\includegraphics[width=\hw]{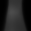}
\includegraphics[width=\hw]{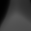}
\includegraphics[width=\w]{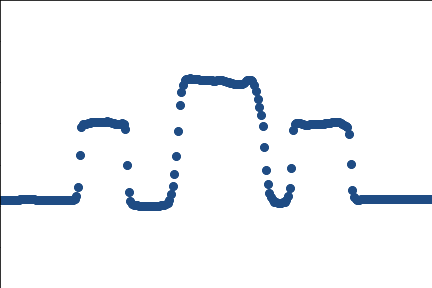}
\end{subfigure}%
\begin{subfigure}{\w}
\caption*{2500}
\includegraphics[width=\w]{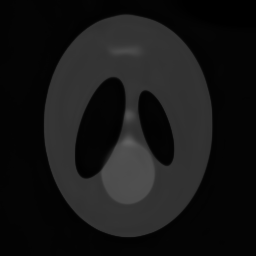}
\includegraphics[width=\hw]{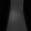}
\includegraphics[width=\hw]{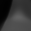}
\includegraphics[width=\w]{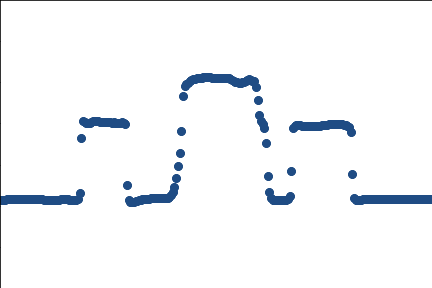}
\end{subfigure}%
\begin{subfigure}{\w}
\caption*{10900}
\includegraphics[width=\w]{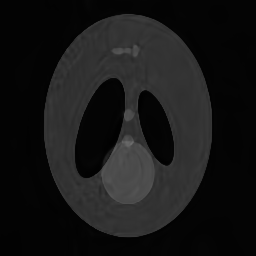}
\includegraphics[width=\hw]{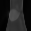}
\includegraphics[width=\hw]{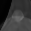}
\includegraphics[width=\w]{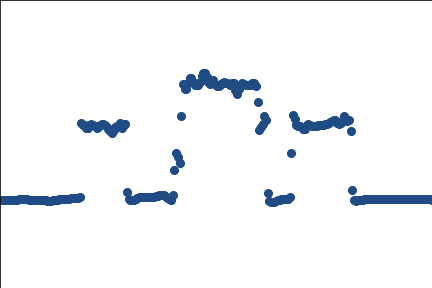}
\end{subfigure}%
\begin{subfigure}{\w}
\caption*{24700}
\includegraphics[width=\w]{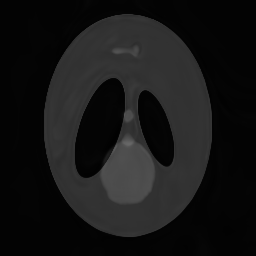}
\includegraphics[width=\hw]{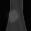}
\includegraphics[width=\hw]{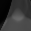}
\includegraphics[width=\w]{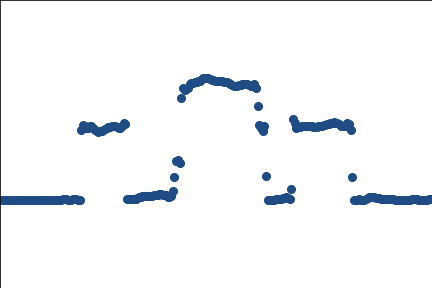}
\end{subfigure}%
\caption{Meshgrid input}
\end{subfigure}
\caption{Denoising with DIP. a) From left to right: Corrupted, ground truth, TV reconstruction, DIP with random input after 300, 1700 (visually best), 4900 (best SSIM), and 20000 iterations. Top row shows the full image, second and third rows the closeups indicated by the red rectangles, last row a cut through the horizontal red line. b) DIP with meshgrid input after 1000, 2500, 10900, 24700 iterations.
In b) we show iterations where the $L^2$ norm of the weights is of similar magnitude as in a).}
\label{fig:dip_experiments}
\end{figure}

\begin{figure}
\centering
\begin{subfigure}[c]{\textwidth}
\centering
\includegraphics[width=12cm]{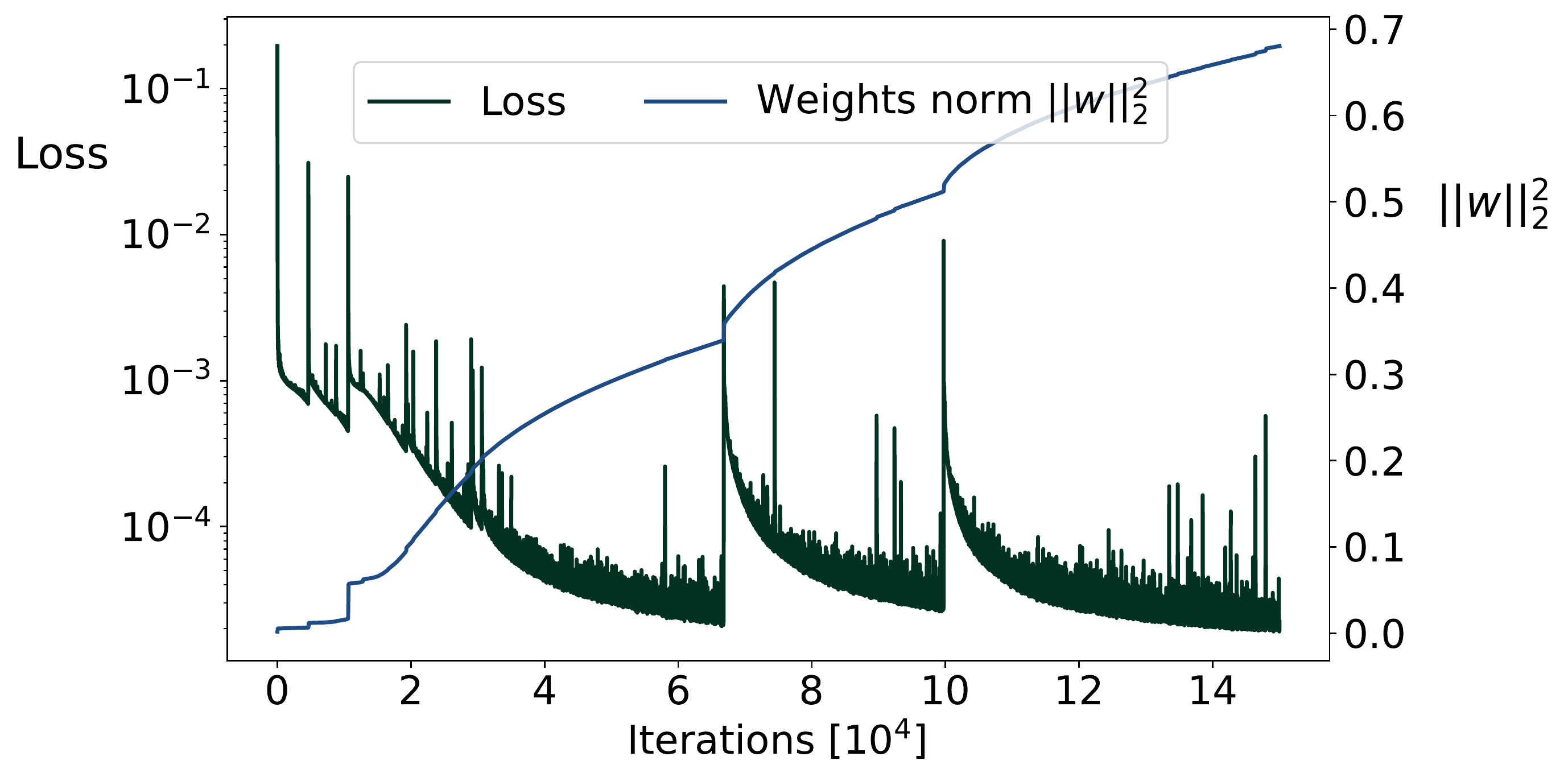}
\caption{Noise input}
\end{subfigure}\\
\begin{subfigure}[c]{\textwidth}
\centering
\includegraphics[width=12cm]{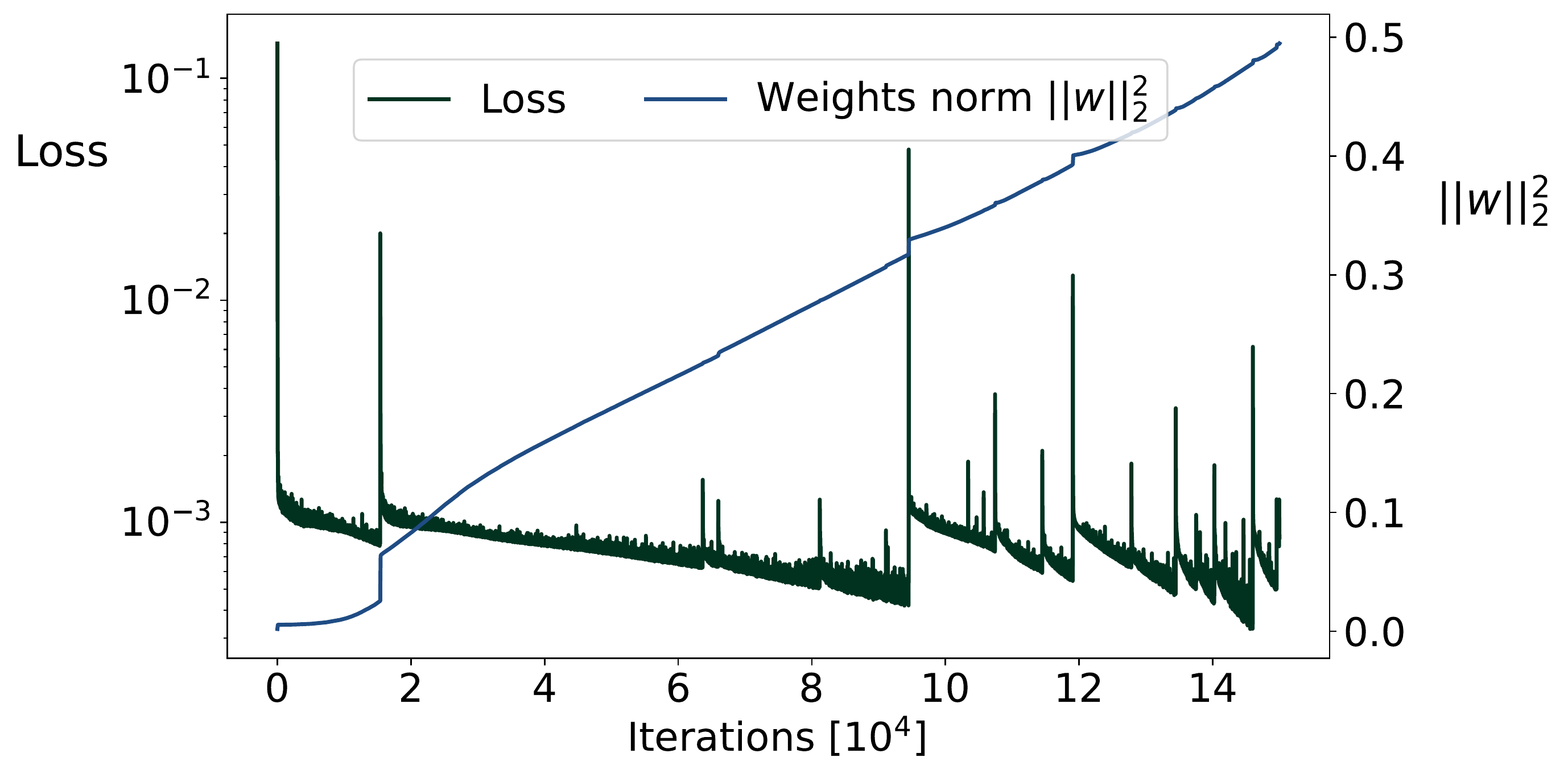}
\caption{Meshgrid input}
\end{subfigure}
\caption{Loss and norm of the weights during optimization. As a loss the standard $L^2$ loss (torch MSE) is used \cite[Section 3, Denoising and generic reconstruction]{deep_image_prior}. If $w\in\R^N$ denotes a vecotrized version of all trainable weights of the network, the depicted squared $L^2$ norm of the weights is computed as $\|w\|_2^2 =\frac{1}{N} \sum_{i=1}^N |w_i|^2$}
\label{fig:dip_graphs}
\end{figure}

\section{Conclusion}\label{sec:conclusion}
Our analysis shows that, in the infinite resolution limit, image representing functions generated by a large class of CNNs are continuous or even $C^1$, preventing them from having sharp edges in the form of jump discontinuities. While this holds for pre-trained CNNs as well as the training process itself, a crucial ingredient in both cases is $L^2$ regularity of the network weights. Our experiments confirm our analytical results, but also show their limitation in case of unbounded weights. Concerning limitations of the practical relevance of our results with respect to the discretization we refer the reader to the following paragraph. 

Given that early stopping or regularization of weights during network training is indispensable for reasons of well-posedness and stability, a practical consequence of our work is that either regularization different from basic $L^2$ penalties or suitably adapted architectures should be used whenever image data or, even more severely, residual noise (e.g., \cite{zhang2017beyond}) is to be approximated by CNNs.
\paragraph{Limitations} As is typical for regularity results in numerical analysis, our theory builds on the convergence for infinite resolution which dictates the extent to which effects are visible in the discrete setting. Thus, state-of-the art CNN methods with a large number of trainable parameters might produce results in the discrete setting that seem to contain sharp edges. Contrary to the infinite dimensional setting, where sharp edges can be defined rigorously as curves of points constituting a discontinuity, a formal definition of sharp edges cannot be given in the discrete setting due to the lack of a meaningful notion of continuity. An edge is perceived by an observer, whenever the contrast of neighboring pixels in an image exceeds a certain threshold which also depends on the surrounding image values. The theoretical results provided in this work guarantee that, if the resolution is progressively increased for a fixed experiment, the generated images will approximate a continuous limit and what is perceived as an edge at first will be continuously smoothed. 

Recent neural network models employ a multitude of heuristics for training, such as perturbing the input with random noise \cite{deep_image_prior}, applying gradient based methods to non-differentiable or non-convex functions, early stopping to avoid overfitting etc., which distort the minimization and cause the trained network to be only weakly related to the loss minimizer. These practices are not covered by our model. Nevertheless, we believe our results are of practical relevance, as shown by the experiments, since they describe the effect of an important building block of CNNs (regularization of weights during training) as the resolution tends towards infinity.

\bibliographystyle{siamplain}
\bibliography{references}

\newpage

\appendix

\section{Regularity Results Used within the Main Proofs}\label{sec:appendix_regularity}
In the following we provide some smaller regularity results which are used primarily to prove \Cref{thm:regularity}.

\begin{lemma}\label{lem:derivative_convolution}
Let $p\in (1,\infty]$, $u\in \BV(\Omega_{\Sigma})$ and $w\in L^p(\Sigma)$. Then $u*w\in W^{1,p}(\Omega)$.
\end{lemma}
\begin{proof}
By \cref{lem_conv_measures} and \Cref{rmk:conv_measure}, we have $u*w, (\D u)*w \in L^p(\Omega)$. A standard computation shows that $\D (u*w)=(\D u)*w$ in the distributional sense.
\end{proof}

\begin{lemma}\label{lem:activation_regularity}
Let $p\in [1,\infty)$, $\sigma: \R \rightarrow \R$ be Lipschitz continuous, and $u\in W^{1,p}(\Omega)$. Define $\sigma(u)$ as the point-wise application of $\sigma$, that is $\sigma(u):\Omega\rightarrow\R$, $\sigma(u)(x)=\sigma(u(x))$. Then $\sigma (u)\in W^{1,p}(\Omega)$.
\end{lemma}
\begin{proof}
By replacing $\sigma$ with $\sigma-\sigma(0)$ if necessary, we can assume $\sigma(0)=0$. In \cite[Theorem 12.69]{leoni2017first} the result is shown for $\Omega=\R^2$. As proven in \cite[Theorem 5.24]{adams2003sobolev}, there exists a linear, bounded extension $E: W^{1,p}(\Omega)\rightarrow W^{1,p}(\R^2)$ such that $E(u)|_\Omega=u$ almost everywhere. Then $\sigma (E(u)) \in W^{1,p}(\R^2)$ implies $\sigma (E(u))|_\Omega = \sigma (u) \in W^{1,p}(\Omega)$.
\end{proof}

\begin{lemma}\label{lem:continuous_weak_derivative}
Let $p\in [1,\infty)$ and $u\in W^{1,p}(\Omega)$ such that $u, \D u\in C(\overline{\Omega})$. Then $u \in C^1(\overline{\Omega})$ with classical derivative $\D u$.
\end{lemma}
\begin{proof}
Let $\phi\in C^\infty_c(\R^2)$ be a mollifier, that is, $\phi\geq 0$, $\int_{\R^2}\phi = 1$ and $\supp (\phi)\Subset B_1(0)$, where $\Subset$ means compactly embedded. Let $\Omega'\Subset\Omega$ be arbitrary. Denoting $\phi_\epsilon(x) = \frac{1}{\epsilon^2}\phi(\frac{x}{\epsilon})$ it is a standard result that 1.) $\phi_\epsilon * \tilde{u}\in C^\infty(\R^2)$, 2.) $\phi_\epsilon * u \rightarrow u$ uniformly in $\Omega'$, and 3.) $\D (\phi_\epsilon * u) =  (\D\phi_\epsilon) * u = \phi_\epsilon*\D u \rightarrow \D u$ uniformly in $\Omega'$. Note that for the second and third property for $\epsilon$ small enough we don't need the zero extension of $u$ since we only consider points in $\Omega'\Subset\Omega$.
Since $C^1(\Omega')$ equipped with the norm $\|v\| = \|v\|_\infty + \| \D v\|_\infty$ is a Banach space, this implies that $u = \lim_{\epsilon\rightarrow 0} \phi_\epsilon * u \in C^1(\Omega')$ and that the classical and weak derivatives of $u$ coincide on $\Omega'$. Because $\Omega'\Subset \Omega$ was arbitrary, we can conclude that $u\in C^1(\Omega)$ and the classical and weak derivatives coincide on $\Omega$ as well. The extension to the boundary follows from the assumption $u, \D u\in C(\overline{\Omega})$ and therefore $u, \D u$ are uniformly continuous (see \cite[p. 618]{evans2010partial}).
\end{proof}

\section{Implementation and Architecture Details}\label{sec:appendix_implementation}

\subsection{End-to-end Imaging}\label{appendix:unet}
\paragraph{Network Architecture}
For this experiment we used a basic U-net architecture which was taken from \cite{im2im_uq_git} and consists of a downsampling and an upsampling path. The downsampling path of the U-net contains 5 modules, each of them consisting of a max pooling operation (this is omitted only in the first module) followed by two consecutive iterations of convolution, batch normalization and a ReLU activation. Similarly, the upsampling path contains 5 modules, the first 4 each consisting of a concatenation with the corresponding activation of the downsampling path (skip connection) and then a bilinear upsampling layer again followed by two consecutive iterations of convolution, batch normalization and ReLU activation. Only the last module consists of a simple convolutional layer. Convolution kernels are of size 3x3, and up- and downsampling is performed with a factor of 2. In order to adjust the resolution of the U-net we introduced a parameter to scale all convolution kernels by the same factor and adjusts padding accordingly. For further details about the implementation we refer to the publicly available source code \cite{cnn_regularity_git}.

\paragraph{Optimization}
Optimization is done using ADAM with a learning rate of 0.001 and a varying weight decay. We use a batch size of 32 and train for 40 epochs on 8000 images.

\paragraph{Data and experimental setup}
We use synthetic data for this experiment. Precisely, the data consists of images containing up to 5 circles of random center, radius and grayscale value. One data sample is generated as described in \Cref{algo:sampling}. For the baseline network we work with images of the size 128x128 and for the double and triple resolution versions the images are of double and triple the baseline resolution per dimension, respectively.

We train the U-net for end to end denoising. The data corruption is additive Gaussian noise with zero mean and a standard deviation of 0.07. We train the network with different values of the weight decay parameter (0, 0.001, 0.01, 0.1, 0.5) of Pytorch's ADAM optimizer. Since as a loss functional we use torch's torch.nn.MSE() with the \emph{mean} reduction, but the weight decay uses $\ell^2$ regularization with the \emph{sum} reduction, we have to adjust the weight decay when changing resolution. Precisely, changing resolution by a factor $\gamma$ in both dimensions has to be compensated by changing the weight decay by a factor $1/\gamma^2$.

\begin{algorithm}
\caption{Sampling data}\label{algo:sampling}
\flushleft\noindent\textbf{Input} Size $n$.
\begin{algorithmic}
\STATE $im \gets (n,n)$ zero image \COMMENT{Given an image shape, initialize a zero image.}
\FOR{$i=1,2,\dots 5$}
\STATE $x_i = (x_i^1,x_i^2)\gets N(0.5,0.5^2)$ \COMMENT{Sample circle center point.}
\STATE $r_i\gets N(0.5,0.5^2)$ \COMMENT{Sample circle radius.}
\STATE $c_i\gets N(0.5,0.5^2)$ \COMMENT{Sample gray scale value.}
\ENDFOR

\STATE $X_i^j \gets \lfloor n x_i^j\rfloor$ \COMMENT{Convert to integer values on pixel grid.}
\STATE $R_i = \lfloor nr_i \rfloor$
\FOR{$i=1,2,\dots 5$}
	\FOR{$k,l = 1,2,\dots n$} 
		\IF{$|X_i-(k,l)|^2 \leq R_i^2$} %
			\STATE $im_{k,l} += c_i$
		\ENDIF
	\ENDFOR
\ENDFOR
\end{algorithmic}
\flushleft\textbf{Output} Sample $im$.
\end{algorithm}

\subsection{Deep Image Prior}\label{appendix:dip}
\paragraph{Network Architecture}
The architecture and optimization we used to generate \Cref{fig:dip_experiments} is exactly the same as used in \cite{deep_image_prior} for the denoising experiment of Figure 4 of this work and can be found in the git repository \cite{deep_image_prior_git}. Nonetheless we describe it here again for the reader's convenience.

We use an encoder-decoder (“hourglass”) architecture with skip connections. We use bilinear upsampling and for downsampling strided convolutions with a stride of 2. LeakyReLU serves as non-linearity and in convolution layers reflection padding is used. A sketch of the network architecture can be found in \Cref{fig_sketch_network}. There, $D_i$ are downsampling, $U_i$ upsampling, and $S_i$ skip connection modules. Each of these modules consists of a specific composition of convolutions, up- or donwsampling layers, nonlinearities, and batch normalization (for details see the supplement to \cite{deep_image_prior}). In the up- and downsampling modules each convolutional layer uses 128 kernels of size 3x3 and in the skip connection modules 4 kernels of size 1x1.

For the version with random input, we used a input with 32 channels sampled from U(0,1/10) and for the version with meshgrid input, the input contains two channels created using np.meshgrid(np.arange(0,M)/(M-1),np.arange(0,N)/(N-1)) where (N,M) is the image shape.

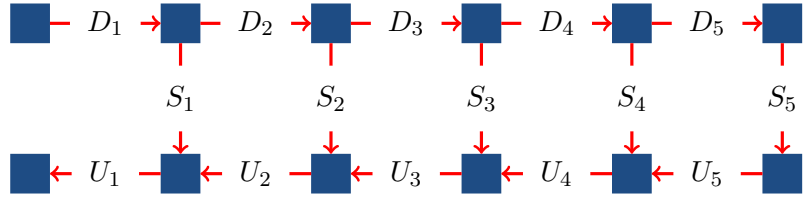
\begin{figure}
\centering
\begin{tikzpicture}
\begin{scope}[every node/.style={rectangle,thick,draw,fill,myblue, minimum size=0.5cm}]
    \node (d1) at (2,2) {};
    \node (d2) at (4,2) {};
    \node (d3) at (6,2) {};
    \node (d4) at (8,2) {};
    \node (d5) at (10,2) {};
    \node (d6) at (12,2) {};
    
    \node (u1) at (2,0) {};
    \node (u2) at (4,0) {};
    \node (u3) at (6,0) {};
    \node (u4) at (8,0) {};
    \node (u5) at (10,0) {};
    \node (u6) at (12,0) {};
\end{scope}
 
\begin{scope}[->,
              every node/.style={fill=white,circle},
              every edge/.style={draw=red,very thick}]
              
    \path [->] (d1) edge node {$D_1$} (d2);
    \path [->] (d2) edge node {$D_2$} (d3);
    \path [->] (d3) edge node {$D_3$} (d4);
    \path [->] (d4) edge node {$D_4$} (d5);
    \path [->] (d5) edge node {$D_5$} (d6);
    
    \path [->] (d2) edge node {$S_1$} (u2);
    \path [->] (d3) edge node {$S_2$} (u3);
    \path [->] (d4) edge node {$S_3$} (u4);
    \path [->] (d5) edge node {$S_4$} (u5);
    \path [->] (d6) edge node {$S_5$} (u6);

    \path [->] (u6) edge node {$U_5$} (u5);
    \path [->] (u5) edge node {$U_4$} (u4);
    \path [->] (u4) edge node {$U_3$} (u3);
    \path [->] (u3) edge node {$U_2$} (u2);
    \path [->] (u2) edge node {$U_1$} (u1);

\end{scope}
\end{tikzpicture}
\caption{Sketch of the network architecture used for the deep image prior experiment.}
\label{fig_sketch_network}
\end{figure}

\paragraph{Optimization}
Optimization is done using ADAM with a learning rate of 0.01. During fitting of the networks we use a noise-based regularization, i.e., at each iteration we perturb the input with an additive normal noise with standard deviation 1/30 which is also proposed by the authors of \cite{deep_image_prior}. Since optimization process tends to destabilize as the loss goes down and approaches a certain value, which manifests as a significant loss increase and a blur in the network output, the optimization loss is tracked and if the loss difference between two consecutive iterations is higher than a threshold we return to parameters from the previous iteration.

\section{Miscellaneous}\label{sec:appendix_miscellaneous}
A component that is also frequently used in CNNs, but was omitted in the main paper for the sake of a clearer presentation, is batch normalization. As we will show now, batch normalization can be included in our framework as well without affecting the results.

\begin{definition}{(Batch normalization)}\label{defin:batch_norm}
Let $u=(u_i)_{i=1}^N\in [L^2(\theta)]^N$. We define the empirical mean $\overline{u}\coloneqq \frac{1}{N|\theta|}\sum_{i=1}^N\int_\theta u_i(x)\; \wrt x$ and variance $\sigma^2\coloneqq \frac{1}{N|\theta|}\sum_{i=1}^N\int_\theta(u_i(x)-\overline{u})^2\; \wrt x$. For an $\epsilon>0$ we define the batch normalization with parameters $\gamma,\beta\in\R^N$ as
\[{\BN}_{\gamma,\beta}(u)_i = \gamma\frac{u_i-\overline{u}}{\sqrt{\sigma^2+\epsilon}}+\beta.\]
\end{definition}
Hence, for fixed $u\in L^2$ the batch normalization ${\BN}_{\gamma,\beta}(u)$ is merely an affine transformation of $u$ and, thus, the regularities $\BV$, $L^p$, $W^{1,p}$, $C^0$, and $C^1$ are maintained under batch normalization. As a result, \Cref{thm:regularity} is still true if batch normalization is included, the only restriction being that it cannot be applied to an input Radon measure as $\sigma$ is not well-defined in this case. But this is only a minor restriction since the regularity is lifted to $L^2$ at the latest after the first convolutional layer.

\begin{lemma}\label{lem:batch_norm}
Let $\theta\subset\R^2$ be a bounded domain, $i=1,2,\dots N$, and $u^h\in \PC^h(\theta)$, $u\in [C(\overline{\theta})]^N$ such that for all $i=1,2,\dots ,N$, $u^h_i\rightarrow u_i$ uniformly, then ${\BN}_{\gamma,\beta}(u^h)_i \rightarrow {\BN}_{\gamma,\beta}(u)_i$ uniformly.
\end{lemma}
\begin{proof}
This result is a direct consequence of the definition of batch normalization using that uniform convergence on a bounded domain implies $L^p$ convergence for all $p$.
\end{proof}
That is, batch normalization maintains uniform convergence and as a consequence also \Cref{thm:CNN_cont} remains true if batch normalization is included. Further the results remain unchanged if the scalar parameters $\gamma$ and $\beta$ are trainable, as long as they are penalized during training as a regularization.

\end{document}